%% file: Comm_priv_multimessage_shuffling.tex
\documentclass[journal,onecolumn,12pt,twoside]{IEEEtran}

\normalsize

\ifCLASSINFOpdf

\else

\fi

\usepackage{url,color,cite}
\usepackage{caption}
\usepackage[pdftex]{graphicx}
\usepackage[cmex10]{amsmath}
\usepackage{amssymb}
\usepackage{amsthm}
\usepackage{mathtools}

\usepackage{float}
\usepackage{epstopdf}
\usepackage{algorithm,algpseudocode}
\usepackage{algpseudocode}
\usepackage{pifont}
\usepackage{subcaption}
\usepackage{cleveref}
\usepackage{multirow}

\newtheorem{lemma}{Lemma}
\newtheorem{theorem}{Theorem}
\theoremstyle{definition}
\newtheorem{definition}{Definition}
\theoremstyle{definition}

\newtheorem{remark}{Remark}
\theoremstyle{definition}
\newtheorem{corollary}{Corollary}
\theoremstyle{definition}

\newcommand{\bbR}{\mathbb{R}}

\newcommand{\calC}{\mathcal{C}}
\newcommand{\calD}{\mathcal{D}}

\newcommand{\calM}{\mathcal{M}}

\newcommand{\calR}{\mathcal{R}}
\newcommand{\calS}{\mathcal{S}}

\newcommand{\calX}{\mathcal{X}}
\newcommand{\calY}{\mathcal{Y}}

\newcommand{\eps}{\epsilon}

\usepackage{subcaption}

\date{}
\title{Multi-Message Shuffled Privacy in Federated Learning}
\author{Antonious M. Girgis and Suhas Diggavi \\ University of California, Los Angeles, USA.\\
Email: amgirgis@g.ucla.edu, suhas@ee.ucla.edu.
}

\begin{document}
\maketitle

\begin{abstract}
\input{abstract}
\end{abstract}

\input{Introduction}

\input{Preliminary}

\input{MainRes}
\input{Proof_outlines}
\input{Numerics}

\nocite{*}
\bibliographystyle{ieeetr}
\bibliography{Priv-tradeoff}

\newpage
\appendices
\input{App_2rr}
\input{App_binary_vector}

\input{App_quantization}
\input{App_Linf_norm}
\input{App_l2_norm}

\input{App_OptRes}

\end{document}

%% file: abstract.tex
We study differentially private distributed optimization under
communication constraints. A server using SGD for optimization,
aggregates the client-side local gradients for model updates using
distributed mean estimation (DME). We develop a communication
efficient private DME, using the recently developed multi-message
shuffled (MMS) privacy framework. We analyze our proposed DME scheme
to show that it achieves the order-optimal
privacy-communication-performance tradeoff resolving an open question
in \cite{pmlr-v162-chen22c}, whether the shuffled models can
improve the tradeoff obtained in Secure Aggregation. This also resolves an
open question on optimal trade-off for private vector sum in the MMS
model. We achieve it through a novel privacy mechanism that
non-uniformly allocates privacy at different resolutions of the local
gradient vectors. These results are directly applied to give
guarantees on private distributed learning algorithms using this for
private gradient aggregation iteratively. We also numerically evaluate
the private DME algorithms.

%% file: Introduction.tex
\section{Introduction}~\label{sec:introduction}

In federated learning (FL) distributed nodes collaborate to build
learning models, mediated by a server\footnote{This is because no
  client has access to enough data to build rich learning models
  locally and we do not want to directly share local data.}. In
particular, they collaboratively build a learning model by solving an
empirical risk minimization (ERM) problem (see
\eqref{eq:problem-formulation} in Section
\ref{sec:preliminary}). Even though local data is not directly
shared, such a collaborative interaction \emph{does not} provide any
privacy guarantee. Therefore, the objective is to solve
\eqref{eq:problem-formulation} while enabling strong privacy
guarantees on local data from the server, but with good
learning performance, \emph{i.e.,} a suitable privacy-learning
performance operating point. Differential Privacy
(DP)~\cite{Calibrating_DP06}, is the accepted theoretical framework
for formal privacy guarantees. Though DP was proposed for central data
storage, the appropriate framework for privacy with distributed
(local) data is local differential privacy
(LDP)~\cite{kasiviswanathan2011can,duchi2013local}, where even the
mediating server is not trusted for privacy. Another
important aspect is that communication in FL occurs in bandwidth
limited (wireless) links, this communication bottleneck can be
significant in modern large-scale machine learning.  The overall goal
of this paper is to develop (both theory and algorithms) for the
\emph{fundamental} privacy-communication-performance trade-off to
solve the ERM in \eqref{eq:problem-formulation} for FL.

\vspace{0.2cm}
\noindent\textbf{Private distributed mean estimation (DME) and
  optimization:} At the core of solving the ERM in
\eqref{eq:problem-formulation} through (stochastic) gradient descent
(SGD) is to aggregate the local gradients, which is equivalent to
finding the (distributed) mean of the users' gradients. Therefore, the
central problem is to study the privacy-communication-performance
trade-off for DME.  Since there are repeated interactions via
iterations of SGD, each exchange leaks information about the local
data, but we need as many steps as possible to obtain a good model;
setting up the tension between privacy and performance. The objective
is to obtain as many such interactions as possible for a given privacy
budget. This is quantified through analyzing the privacy of the
composition of privacy mechanisms as a function of the number of
iterations, and such tight analyses have been developed for
composition in \cite{abadi2016deep,mironov2017renyi}. We use
compositional bounds from
\cite{girgis2021renyi-CCS,feldman2022stronger} in conjunction with our
new private DME mechanisms to obtain the
privacy-communication-performance trade-off for solving
\eqref{eq:problem-formulation} (see Theorem \ref{thm:app_Opt}).

\vspace{0.2cm}
\noindent\textbf{Privacy frameworks:} A strong privacy guarantee
includes an untrustworthy server, and to guarantee this, in LDP each
client randomizes its interactions with the server from whom the data
is to be kept private (\emph{e.g.,} see implementations
~\cite{erlingsson2014rappor,microsoft}). The fundamental
privacy-communication-performance trade-offs of LDP mechanisms for
private DME have been recently studied
\cite{chen2020breaking,girgis2021shuffled-aistats}. We study a new
approach to the privacy-communication-performance trade-off (see
Theorems~\ref{thm:l_inf_vector_ldp}, \ref{thm:l_2_vector_ldp} which
are also order optimal, and we adapt it for other privacy frameworks
below.

LDP mechanisms suffer from poor performance in comparison with the
central DP
mechanisms~\cite{kasiviswanathan2011can,kairouz2016discrete}. In
order to overcome this, two privacy frameworks have been advocated,
which enable significantly better privacy-performance trade-offs by
amplifying privacy: {\sf (i)} \emph{Secure Aggregation (SecAgg)}: This
is a secure sum protocol \cite{Bell20} which only allows the server
to see the sum of vectors, and not individual ones. {\sf (ii)}
\emph{Shuffled model}: Each user sends her private message to a secure
shuffler that randomly permutes all the received messages before
forwarding them to the server
\cite{erlingsson2019amplification,cheu2019distributed}. The extension
to this is the \emph{multi-message shuffled (MMS)} model, where there
are multiple parallel shuffled models as above. In
\cite{Balle-MMS-CCS20,Ghazi-Eurocrypt21} it has been shown that one
can get significantly better trade-offs with such multi-message
shuffled (MMS) models. In this paper we focus on such multi-message shuffled (MMS)
privacy models.

\vspace{0.2cm}
\noindent\textbf{Contributions:} Motivated by these discussions, we make the following contributions.

\begin{itemize}

\item In \cite{pmlr-v162-chen22c}, a (order-wise) fundamental
trade-off for privacy-communication-performance was established for
DME for the SecAgg privacy framework, and an open question was posed on this
trade-off for the shuffled models. In this paper we resolve this
question through a fundamental privacy-communication-performance
trade-off for DME in the (multi-message) shuffled (MMS) models, for
\emph{all} regimes; we believe ours is the first scheme to achieve the
complete optimal trade-off (see Theorems~\ref{thm:l_inf_vector_shuffle}, \ref{thm:l_2_vector_shuffle}) which matches lower bound (see Theorem~\ref{thm:L_2_lower_bound_central}). Furthermore, we show that our MMS requires less amount of communication per client than used in the SecAgg to achieve the same order of MSE (See Remark~\ref{rem:comp_secagg}).

\item In
  \cite{Balle-MMS-CCS20,Ghazi-Eurocrypt21},
  it was shown that for computing \emph{scalar sum} in \emph{multi-message
  shuffled} (MMS) models can fundamentally achieve trade-off points that
  single-message shuffled models cannot. The optimal trade-off for
  computing \emph{vector sum} is an open question, and the only known
  result \cite{Cheu-ICLR22} has communication \emph{per-user}
  growing as $\mathcal{O}(d\sqrt{n})$, where $n$ is number of users and $d$ is the vector dimension. In this paper we establish the fundamental
  privacy-communication-performance trade-off for computing
  \emph{vector sum} in the multi-message shuffled model (see Theorems~\ref{thm:l_inf_vector_shuffle}, \ref{thm:l_2_vector_shuffle}) for all trade-off regimes, which order-wise is better than the results in \cite{Cheu-ICLR22}. In doing so, we also resolve this
  trade-off for all regimes in the scalar case (see Remark~\ref{rem:scalar_MMS}).

\item  Our scheme when applied to LDP, also achieves the optimal
  trade-off for this privacy framework (see
  Theorems~\ref{thm:l_inf_vector_ldp}, \ref{thm:l_2_vector_ldp}),
  similar to \cite{chen2020breaking,girgis2021shuffled-aistats} and (order-wise) better performance than \cite{DDG-ICML21} when applied to LDP. (see
  Remark~\ref{rem:secAgg}). Since the idea of
  \cite{DDG-ICML21} was used as a primitive in \cite{pmlr-v162-chen22c}, we
  can plug in our method to potentially improve the trade-off in their
  scheme.

\item  We use the results for optimal private DME to analyze
  privacy-convergence trade-offs of the DP-SGD algorithm (similar to
  algorithms in \cite{girgis2021shuffled-aistats,Cheu-ICLR22} in
  Theorem \ref{thm:app_Opt}.

  \item  In Section \ref{sec:numerics}, we evaluate the performance of our proposed algorithms for scalar and vector private DME.

\end{itemize}

  The core technical idea that enables these results is the
  following. Suppose each client $i$ holds a real vector
  $\mathbf{x}_i$, and we want to privately compute the sum
  $\sum_i\mathbf{x}_{i}$. First we devise a co-ordinate sampling
  mechanism related to the target communication desired, independently
  for each client; then we compute the private scalar sum
  $\sum_{i\in\mathcal{A}_k}\mathbf{x}_i[k]$, where $\mathbf{x}_i[k]$
  is the $k$-th co-ordinate, and $\mathcal{A}_k$ is the set of clients
  that sampled the $k$-th co-ordinate. We can express
  $\mathbf{x}_i[k]=0.\mathbf{b}_i^{(1)}\mathbf{b}_i^{(2)}\ldots,\mathbf{b}_i^{(m)}\ldots$ in binary form\footnote{We have written this for
    $\mathbf{x}_i[k]\in[0,1]$ but can be easily extended to any
    bounded values, \emph{i.e.,} $\|\mathbf{x}_i\|_{\infty}\leq
    r_{\infty}$.}, where $\mathbf{b}_i^{(j)}\in\{0,1\}$. For privacy, we
  randomize each bit through a binary randomized response \cite{warner1965-RR},
  but we randomize each bit with a different privacy budget, so that we meet an
  overall privacy budget. This careful choice of such non-uniform
  randomization is key to our method. Moreover, for communication
  constraints we represent it with finite $m$ bits (see more details
  in Section~\ref{sec:MainRes}). We can either use this overall
  randomization as is, for LDP, or send each bit through a separate
  shuffler for multi-message shuffling (MMS). Then by carefully
  accounting for the composition using RDP, we obtain our privacy
  guarantees and performance (see Lemmas~\ref{lem:Priv-alloc},
  \ref{lem:MSE-perf}). This simple mechanism yields explicit bounds
  for the complete trade-off and forms the core of our solution.

 \subsection{Related Work}
We give the most relevant work related to the paper and review some of their connections to our work.

\paragraph{Private DME:} In \cite{chen2020breaking,girgis2021shuffled-aistats} the privacy-communication-performance tradeoff were studied both through schemes as well as lower bounds for the local DP model. \cite{chen2020breaking} established the order optimal private DME under local DP model for bounded $\ell_2$-norm vectors. \cite{girgis2021shuffled-aistats} established order optimal private DME for local DP for bounded $\ell_{\infty}$-norm and separately for bounded $\ell_2$-norm vectors. It also extended its use in the single-shuffled model and private optimization framework (see below). In~\cite{guo2022interpolated,chaudhuri2022privacy}, a family of communication-efficient mechanisms is proposed under LDP constraints in federated learning.

In the multi-message shuffled (MMS) model, the private \emph{scalar} DME was
studied in \cite{Balle-MMS-CCS20,Ghazi-Eurocrypt21}, where order optimal strategies were established. The private vector DME has received less
attention, with the exception of \cite{Cheu-ICLR22}. Our private
vector DME result in Theorem \ref{thm:l_2_vector_shuffle} improves the
privacy-communication-performance order-wise over it. In \cite{DDG-ICML21,pmlr-v162-chen22c}, the privacy-communication-performance trade-off in the SecAgg privacy model was studied. In particular, using ideas from compressive sensing, \cite{pmlr-v162-chen22c} established an order-optimal private DME for SecAgg.

\paragraph{Private optimization in the shuffled model:}
There has been a lot of work on private optimization in the local
model, see \cite{agarwal2018cpsgd,girgis2021shuffled-aistats} and references
therein. We will focus on private optimization in the shuffled model,
where there is relatively less work.  Recently~\cite{ESA}
and~\cite{girgis2021shuffled-aistats,girgis2021shuffled-jsait} have
proposed DP-SGD algorithms for federated learning, in the shuffled
model, where at each iteration, each client applies an LDP mechanism
on the gradients. \cite{girgis2021renyi_neurips} studied a private
optimization framework using RDP and additionally evaluated
subsampling (of clients) in the shuffled model. The approach in
~\cite{ESA} was to send full-prevision gradients without compression,
but ~\cite{girgis2021shuffled-aistats,girgis2021renyi_neurips} did use
compression for the gradients. These methods achived certain optimal
privacy-communication-performance operating points, but not in all
regimes. The use of RDP for establishing compositional bounds for
interactive optimization was studied in
\cite{girgis2021renyi-CCS,feldman2022stronger}, which is used in
establishing the privacy bounds for iterative stochastic
optimization. All these were for the single-shuffle model. For the
multi-message shuffled (MMS) model, private optimization was studied
in \cite{Cheu-ICLR22}, which at its core used a private vector DME
with MMS. As explained earlier, our private vector DME is orderwise
better than this scheme, and if we plug our scheme into the standard
convergence analyses for optimization, we obtain better results as
also given in Appendix \ref{app:OptRes}.

\vspace{0.2cm}
  \noindent\textbf{Paper organization:} We formulate the problem,
  establish notation and some preliminary results in Section
  \ref{sec:preliminary}. We present an overview of the algorithms and
  the main theoretical results in Section \ref{sec:MainRes}. The
  technical proof ideas are outlined in Section
  \ref{sec:Pf_Outlines}. Some numerical results are presented in Section
  \ref{sec:numerics}. The proof details are given
  in the appendices.

%% file: Preliminary.tex
\section{Preliminaries}~\label{sec:preliminary}
We give privacy definitions in Section~\ref{sec:privacy_defns} and the binary randomized response in Section~\ref{sec:binary_rr}.

\subsection{Privacy Definitions}~\label{sec:privacy_defns}
In this section, we define different privacy notions that we will use in this paper: local differential privacy (LDP), central different privacy (DP), and Renyi differential privacy (RDP). We also give standard results on privacy composition as well as conversion between privacy notions.

\begin{definition}[Local Differential Privacy - LDP~\cite{kasiviswanathan2011can}]~\label{defn:LDPdef}
For $\epsilon_0\geq0$, a randomized mechanism $\calR:\calX\to\calY$ is said to be $\eps_0$-local differentially private (in short, $\eps_{0}$-LDP), if for every pair of inputs $d,d'\in\calX$, we have 
\begin{equation}~\label{app:ldp-def}
\Pr[\calR(d)\in \calS] \leq e^{\eps_0}\Pr[\calR(d')\in \calS], \qquad \forall \calS\subset\calY.
\end{equation}
\end{definition}
Let $\calD=\lbrace d_1,\ldots,d_n\rbrace$ denote a dataset comprising $n$ points from $\mathcal{X}$. We say that two datasets $\calD=\lbrace d_1,\ldots,d_n\rbrace$ and $\calD^{\prime}=\lbrace d_1^{\prime},\ldots,d_n^{\prime}\rbrace$ are neighboring (and denoted by $\calD\sim\calD'$) if they differ in one data point, i.e., there exists an $i\in[n]$ such that $d_i\neq d'_i$ and for every $j\in[n],j\neq i$, we have $d_j=d'_j$.
\begin{definition}[Central Differential Privacy - DP \cite{Calibrating_DP06,dwork2014algorithmic}]\label{defn:central-DP}
For $\epsilon,\delta\geq0$, a randomized mechanism $\calM:\calX^n\to\calY$ is said to be $(\epsilon,\delta)$-differentially private (in short, $(\epsilon,\delta)$-DP), if for all neighboring datasets $\calD\sim\calD^{\prime}\in\calX^{n}$ and every subset $\calS\subseteq \calY$, we have
\begin{equation}~\label{dp_def}
\Pr\left[\calM(\calD)\in\calS\right]\leq e^{\eps_0}\Pr\left[\calM(\calD^{\prime})\in\calS\right]+\delta.
\end{equation}
\end{definition}

\begin{definition}[$(\alpha,\epsilon(\alpha))$-RDP (Renyi Differential Privacy)~\cite{mironov2017renyi}]\label{defn:RDP}
A randomized mechanism $\calM:\calX^n\to\calY$ is said to have $\epsilon(\alpha)$-Renyi differential privacy of order $\alpha\in(1,\infty)$ (in short, $(\alpha,\epsilon(\alpha))$-RDP), if for any neighboring datasets $\calD\sim\calD'\in\calX^n$, we have that $D_{\alpha}(\calM(\calD)||\calM(\calD'))\leq \epsilon(\alpha)$, where $D_{\alpha}(P||Q)$ denotes the Renyi divergence between two distributions $P$ and $Q$ defined by: 
\begin{equation}
D_{\alpha}(P||Q)=\frac{1}{\alpha-1}\log\left(\mathbb{E}_{\theta\sim Q}\left[\left(\frac{P(\theta)}{Q(\theta)}\right)^{\alpha}\right]\right),
\end{equation}
\end{definition}
The RDP provides a tight privacy accounting of adaptively composed
mechanisms.  The following result states that if we adaptively compose
two RDP mechanisms with the same order, their privacy parameters add
up in the resulting mechanism.

\begin{lemma}[Adaptive composition of RDP~{\cite{mironov2017renyi}}]\label{lemm:compostion_rdp} 
For any $\alpha>1$, let $\calM_1:\calX\to \calY_1$ be a $(\alpha,\epsilon_1(\alpha))$-RDP mechanism and $\calM_2:\calY_1\times \calX\to \calY$ be a $(\alpha,\epsilon_2(\alpha))$-RDP mechanism. Then, the mechanism defined by $(\calM_1,\calM_2)$ satisfies $(\alpha,\epsilon_1(\alpha)+\epsilon_2(\alpha))$-RDP. 
\end{lemma}

We use the following result for converting the RDP guarantees of a
mechanism to its DP guarantees.
\begin{lemma}[From RDP to DP~\cite{canonne2020discrete,Borja_HypTest-RDP20}]\label{lem:RDP_DP} 
Suppose for any $\alpha>1$, a mechanism $\calM$ is $\left(\alpha,\epsilon\left(\alpha\right)\right)$-RDP. For any $\delta>0$, the mechanism $\calM$ is $\left(\epsilon_{\delta},\delta\right)$-DP, where $\epsilon_{\delta}$ is given by: 
\begin{equation*}
\begin{aligned}
 & \epsilon_{\delta} = \min_{\alpha} \epsilon\left(\alpha\right)+\frac{\log\left(1/\delta\right)}{\alpha-1}+\log\left(1-1/\alpha\right)\\
\end{aligned}
\end{equation*}
\end{lemma}
\subsection{Binary Randomized Response (\textsl{2RR})}~\label{sec:binary_rr}
The binary randomized response (\textsl{2RR}) is one of the most popular private mechanism that first proposed in~\cite{warner1965-RR}. We present an unbiased version of the \textsl{2RR} mechanism in Algorithm~\ref{algo:2RR} whose input is a bit $b\in\lbrace 0,1\rbrace$ and the output can take one of two values $\lbrace \frac{-p}{1-2p},\frac{1-p}{1-2p} \rbrace$, where $p$ controls privacy-accuracy trade-offs. Furthermore, we present the mean square error (MSE) of the \textsl{2RR} in the following Theorem.

\begin{theorem}~\label{thm:binary_mse} For any $p\in[0,1/2)$, the \textsl{2RR} is $\epsilon_0$-LDP, where $\epsilon_0=\log\left(\frac{1-p}{p}\right)$. The output $y$ of the \textsl{2RR} mechanism is an unbiased estimate of $b$ with bounded MSE:
\begin{equation}~\label{eqn:mse_binary_sum}
\mathsf{MSE}^{\textsl{2RR}}=\sup_{ b\in\lbrace0,1\rbrace }\mathbb{E}\left[\|b-y\|_2^2\right] = \frac{p(1-p)}{(1-2p)^2}.
\end{equation}
\end{theorem}
For completeness, we present the proof of Theorem~\ref{thm:binary_mse} in Appendix~\ref{app:2rr_mse}.

\begin{algorithm}[t]
\caption{: Local Randomizer $\mathcal{R}^{\textsl{2RR}}_p$ }\label{algo:2RR}
\begin{algorithmic}[1]
\State \textbf{Public parameter:} $p$
\State \textbf{Input:} $b\in\lbrace 0,1\rbrace$.
\State Sample $\gamma\gets \text{Ber}\left(p\right)$
\If{$\gamma == 0$}
\State $y = \frac{b-p}{1-2p}$
\Else
\State $y=\frac{1-b-p}{1-2p}$
\EndIf
\State \textbf{Return:} The client sends $y$.
\end{algorithmic}
\end{algorithm}

\section{Problem formulation}~\label{sec:problem_formulation}
 We consider a distributed private learning setup comprising a set of $N$ clients, where the $i$th client has a data set $\mathcal{V}_i$ for $i\in\left[N\right]$. Let $\mathcal{D}=\left(\mathcal{V}_1,\ldots,\mathcal{V}_N\right)$ denote the entire training dataset, with $\mathcal{V}_i$ held locally by user $i$. The clients are connected to an untrusted server in order to solve the following empirical risk minimization (ERM) problem
\begin{equation}\label{eq:problem-formulation}
\min_{\theta\in\mathcal{C}} \Big( F(\theta,\mathcal{D}) := \frac{1}{N}\sum_{i=1}^{N} \sum_{\mathbf{v}\in\mathcal{V}_i} f(\theta,\mathbf{v}) \Big),
\end{equation}
where $\mathcal{C}\subset \mathbb{R}^d$ is a closed convex set, $\mathbf{v}\in\mathcal{V}$, and $f:\calC\times\mathcal{V}\to\bbR$, is the loss function. Our goal is to construct a global learning model $\theta$ via stochastic gradient descent (SGD) while preserving privacy of individual data points in the training dataset $\mathcal{D}$ by providing strong DP guarantees. SGD can be written as
\[
\theta_{t+1} \leftarrow \theta_t - \eta_t \frac{1}{n}\sum_{i\in\mathcal{I}} \mathcal{R}(\nabla f_i(\theta_{t})),
\]
where $\mathcal{R}$ is the local randomization mechanism and
$\mathcal{I}$ are the indices of the clients partipating in that round
of SGD, with $n=|\mathcal{I}|$. Therefore, at each iteration the server does  distributed mean estimation (DME) of the gradients $\frac{1}{n}\sum_{i\in\mathcal{I}} \mathcal{R}(\nabla f_i(\theta_{t}))$, and we want it to be done privately and communication-efficiently. To isolate this problem we define DME under privacy and communication constraints. Suppose we have a set of $n$ clients. Each client has has a $d$ dimensional vector $\mathbf{x}_i\in\mathcal{X}$ for $i\in[n]$, where $\mathcal{X}\subset\mathbb{R}^{d}$ denotes a bounded subset of all possible inputs. For example, $\mathcal{X}\triangleq \mathbb{B}^{d}_2(r_2)$ denotes the $d$ dimensional ball with radius $r_2$, i.e., each vector $\mathbf{x}_i$ satisfies $\|\mathbf{x}_i\|_2\leq r_2$ for $i\in[n]$. Furthermore, each client has a communication budget of $b$-bits. The clients are connected to an (untrusted) server that wants to estimate  $\overline{\mathbf{x}}=\sum_{i=1}^{n}\mathbf{x}_i$.

\vspace{0.2cm}
\noindent\textbf{Privacy frameworks:} We assume an untrusted server, under  two different privacy models: {\sf (i)} Local DP (LDP) model {\sf (ii)} Multi-message shuffled (MMS) model.

\noindent\textbf{LDP-model}: We design two mechanisms: {\sf (i)} client-side mechanism $\mathcal{R}:\mathcal{X}\to\mathcal{Y}$ and {\sf (ii)} Server aggregator $\mathcal{A}:\mathcal{Y}^{n}\to \mathbb{R}^{d}$. The local mechanism $\mathcal{R}$ takes an input $\mathbf{x}_i\in\mathcal{X}$ and generates a randomized output $\mathbf{y}_i\in\mathcal{Y}$. The local mechanism $\mathcal{R}$ satisfies privacy and communication constraints as follows. The output $\mathbf{y}_i=\mathcal{R}\left(\mathbf{x}_i\right)$ can be represented using only $b$-bits. The mechanism $\mathcal{R}$ satisfies $\epsilon_0$-LDP (see Definition~\ref{defn:LDPdef}). Each client sends the output $\mathbf{y}_i$ directly to the server, which applies the aggregator $\mathcal{A}$ to estimate the mean $\hat{\mathbf{x}}=\mathcal{A}\left(\mathbf{y}_1,\ldots,\mathbf{y}_n \right)$ such that the estimated mean $\hat{\mathbf{x}}$ is an unbiased estimate of the true mean $\overline{\mathbf{x}}$.

\noindent\textbf{MMS-model}: The \emph{single} shuffle model is
similar to the local DP model but with a secure
shuffler (permutation) which anonymizes the clients to the server;
shuffling can amplify the privacy of the algorithm. Precisely, the
shuffle model consists of three parameters
$\left(\mathcal{R},\mathcal{S},\mathcal{A}\right)$: {\sf (i)}
  \emph{Encode:} a set of local mechanisms
$\mathcal{R}^{(k)}:\mathcal{X}\to\mathcal{Y}, k=1,\ldots,m$ each
similar to the local DP model. Each client sends the $m$ outputs
$\mathbf{y}_i^{(k)}, k=1,\ldots,m$, where
$\mathbf{y}_i^{(k)}\in\mathcal{Y}$, to the secure
shufflers. {\sf (ii)} \emph{Multi-message Shuffle:} a single secure shuffler
$\mathcal{S}_k:\mathcal{Y}^{n}\to \mathcal{Y}^{n}$ receives $n$
outputs $\mathbf{y}_i^{(k)}, i=1,\ldots,n$ after applying the local
mechanism $\mathcal{R}^{(k)}$ on each input
$\mathbf{x}_1,\ldots,\mathbf{x}_n$ and generates a random permutation
$\pi^{(k)}$ of the received messages. The multi-message shuffle is a
parallel set of $m$ single-message shufflers
$\{\mathcal{S}_k\}$. {\sf (iii)} \emph{Analyze:} the server receives the $m$
shufflers' outputs and applies the aggregator
$\mathcal{A}:\mathcal{Y}^{nm}\to \mathbb{R}^{d}$ to estimate the mean
$\hat{\mathbf{x}}=\mathcal{A}\left(\mathbf{y}_{\pi^{(k)}(1)},\ldots,\mathbf{y}_{\pi^{(k)}(n)},k=1,\ldots,m\right)$.
We say that the shuffled model is $\left(\epsilon,\delta\right)$-DP if
the view of the output of the multi-message shuffler
$\left(\mathbf{y}_{\pi^{(k)}(1)},\ldots,\mathbf{y}_{\pi^{(k)}(n)},k=1,\ldots,m \right)$
satisfies $\left(\epsilon,\delta\right)$-DP.

In the two privacy models, the performance of the estimator $\hat{\mathbf{x}}$ is measured by the expected loss:
\begin{equation}~\label{eqn:loss}
\mathsf{MSE}= \sup_{\lbrace\mathbf{x}_i\in\mathcal{X}\rbrace}\mathbb{E}\left[\|\hat{\mathbf{x}}-\overline{\mathbf{x}}\|_2^2\right],
\end{equation}
where the expectation is taken over the randomness of the private
mechanisms. Hence, our goal is to design communication-efficient and
private schemes to generate an unbiased estimate of the true mean
$\overline{x}$ while minimizing the expected loss~\eqref{eqn:loss}. 
We  study the DME for bounded $\ell_{\infty}$-norm  \emph{i.e.,} $\|\mathbf{x}_i\|_{\infty}\leq r_{\infty}$ for all $i\in[n]$ and for bounded $\ell_2$-norm vectors where $\|\mathbf{x}_i\|_{2}\leq r_{2}$.

%% file: MainRes.tex
\section{Overview and main theoretical results}~\label{sec:MainRes}
\vspace{-0.2in}

\begin{algorithm}[t]
\caption{: Local Randomizer $\mathcal{R}^{\ell_{\infty}}_{v,m,s}$ }\label{algo:l_inf}
\begin{algorithmic}[1]
\State \textbf{Public parameter:} Privacy budget $v$, communication levels $m$, and communication coordinates per level $s$.
\State \textbf{Input:} $\mathbf{x}_i\in \mathbb{B}^{d}_{\infty}\left(r_{\infty}\right)$.
\State $\mathbf{z}_i \gets \left(\mathbf{x}_i+r_{\infty}\right)/2r_{\infty}$
\State $\mathbf{z}^{(0)}_i\gets 0$
\For{$k=1,\ldots,m-1$}
\State $\mathbf{b}_i^{(k)}\gets \lfloor 2^{k}(\mathbf{z}_i-\mathbf{z}^{(k-1)}_i)\rfloor$
\State $v_k \gets \frac{4^{\frac{-k}{3}}}{\left(\sum_{l=1}^{m-1}4^{\frac{-l}{3}}+4^{\frac{-m+1}{3}}\right)} v$
\State $p_k\gets \frac{1}{2}\left(1-\sqrt{\frac{v_k^2/s^2}{v_k^2/s^2+4}}\right)$
\State $\mathcal{Y}_i^{(k)}\gets \mathcal{R}^{\text{Bin}}_{p_k,s}(\mathbf{b}_i^{(k)})$
\State $\mathbf{z}^{(k)}_i\gets \mathbf{z}^{(k-1)}_i+\mathbf{b}_i^{(k)}2^{-k}$
\EndFor
\State Sample $\mathbf{u}_i\gets \mathsf{Bern}\left(2^{m-1}\left(\mathbf{z}_i-\mathbf{z}^{(m-1)}_i\right)\right)$
\State $v_m \gets \frac{4^{\frac{-m+1}{3}}}{\left(\sum_{l=1}^{m-1}4^{\frac{-l}{3}}+4^{\frac{-m+1}{3}}\right)}v$
\State $p_m\gets \frac{1}{2}\left(1-\sqrt{\frac{v_m^2/s^2}{v_m^2/s^2+4}}\right)$
\State $\mathcal{Y}_i^{(m)}\gets \mathcal{R}^{\text{Bin}}_{p_m,s}(\mathbf{u}_i)$
\State \textbf{Return:} The client sends $\mathcal{Y}_i\gets\left\{\mathcal{Y}_i^{(1)},\ldots,\mathcal{Y}_i^{(m)}\right\}$.
\end{algorithmic}
\end{algorithm}

In this section we give an overview of our algorithmic solution for
private DME and the theoretical guarantees for two important cases of
boundedness constraints on the individual vectors. We consider the
private DME of bounded $\ell_{\infty}$-norm vectors in Section
\ref{sec:l_inf_norm} and that for bounded $\ell_{2}$-norm vectors in
Section \ref{sec:l_2_norm}. We will use these results to provide the guarantees for solving the trade-off for the ERM problem of \eqref{eq:problem-formulation} in the Appendix \ref{app:OptRes} (Theorem \ref{thm:app_Opt}).

\begin{algorithm}[t]
\caption{: Analyzer $\mathcal{A}^{\ell_{\infty}}$ }\label{algo:l_inf_server}
\begin{algorithmic}[1]
\State \textbf{Inputs:} $\mathcal{Y}_1,\ldots,\mathcal{Y}_n$, where $\mathcal{Y}_i=\left\{\mathcal{Y}_i^{(1)},\ldots,\mathcal{Y}_i^{(m)}\right\}$ is a set of $m$ sets.
\For {$k=1,\ldots,m-1$}
\State $\hat{\mathbf{b}}^{(k)}\gets\mathcal{A}^{\text{Bin}}\left(\mathcal{Y}_1^{(k)},\ldots,\mathcal{Y}_n^{(k)}\right)$
\EndFor
\State $\hat{\mathbf{u}}\gets\mathcal{A}^{\text{Bin}}\left(\mathcal{Y}_1^{(m)},\ldots,\mathcal{Y}_n^{(m)}\right)$
\State $\hat{\mathbf{z}}\gets \sum_{k=1}^{m-1}\hat{\mathbf{b}}^{(k)}2^{-k}+\hat{\mathbf{u}}2^{-m+1}$ 
\State \textbf{Return:} The server returns $\hat{\mathbf{x}}\gets 2r_{\infty}\hat{\mathbf{z}}-r_{\infty}$.
\end{algorithmic}
\end{algorithm}

\subsection{Bounded $\ell_{\infty}$-norm vectors}~\label{sec:l_inf_norm}
We consider privately computing $\sum_{i=1}^n\mathbf{x}_i$ where
$i$th client has a vector $\mathbf{x}_i$ such that
$\|\mathbf{x}_i\|_{\infty}\leq r_{\infty}$ for $i\in[n]$. For ease of
operation, we will scale each vector such that each coordinate becomes
bounded in range $\left[0,1\right]$, and then reverse it at the
end. That is, each client scales her vector $\mathbf{x}_i$ as follows:
$\mathbf{z}_i=\frac{\mathbf{x}_i+r_{\infty}}{2r_{\infty}}$, where the
operations are done coordinate-wise. Thus, we have that
$\mathbf{z}_{i}[j]\in[0,1]$ for all $j\in[d]$ and $i\in[n]$, where
$\mathbf{z}_{i}[j]$ denotes the $j$th coordinate of the vector
$\mathbf{z}_{i}$. Observe that the vector $\mathbf{z}_i$ can be decomposed into a weighted summation of binary vectors as follows:
\begin{equation}
\label{eq:BinExp}
\mathbf{z}_i = \sum_{k=1}^{\infty} \mathbf{b}_{i}^{(k)}2^{-k},
\end{equation}  
where $\mathbf{b}_i^{(k)}\in\lbrace 0,1\rbrace^{d}$ for all $k\geq 1$. Each client can recursively construct $\mathbf{b}_i^{(k)}$ as follows. Let $\mathbf{z}_i^{(0)}=\mathbf{0}$ and $\mathbf{z}_i^{(k)}=\sum_{l=1}^{k}\mathbf{b}_i^{(l)}2^{-l}$. Hence, $\mathbf{b}_i^{(k)}=\lfloor 2^{k}\left(\mathbf{z}_i-\mathbf{z}^{(k-1)}_i\right)\rfloor$ for $k\geq 1$.

 To make our mechanism communication efficient, each client
 approximates the vector $\mathbf{z}_i$ by using the first $m$ binary
 vectors $\lbrace \mathbf{b}_{i}^{(k)}: 1\leq k\leq m\rbrace$. Note
 that the first $m$ binary vectors together give an approximation to the real
 vector $\mathbf{z}_i$ with error
 $\|\mathbf{z}_i-\mathbf{z}^{(m)}_i\|_2^2\leq d/4^{m}$, where
 $\mathbf{z}^{(m)}_i=\sum_{k=1}^{m}\mathbf{b}_i^{(k)}2^{-k}$. However,
 this mechanism creates a biased estimate of $\mathbf{z}_i$. Hence, to
 design an unbiased mechanism, the client approximates the vector
 $\mathbf{z}_i$ using the first $m-1$ binary vectors $\lbrace
 \mathbf{b}_{i}^{(k)}:1\leq k\leq m-1\rbrace$ of the binary
 representation above and the last binary vector ($\mathbf{u}_i$) is
 reserved for unbiasness as follows:
 \begin{equation}
   \label{eq:real-unbias}
\mathbf{u}_i[j]=\mathsf{Bern}\left(2^{m-1}(\mathbf{z}_{i}[j]-\mathbf{z}_{i}^{(m-1)}[j])\right),
\end{equation}
where $\mathbf{z}^{(m-1)}_i=\sum_{k=1}^{m-1}\mathbf{b}_{i}^{(k)}2^{-k}$ and $\mathsf{Bern}(p)$ denotes Bernoulli random variable with bias $p$. Note that when each client sends the $m$ binary vectors $\lbrace \mathbf{b}_i^{(k)}:1\leq k\leq m-1\rbrace\bigcup \lbrace \mathbf{u}_{i}\rbrace$, the server can generates an unbiased estimate to the mean $\overline{z}=\frac{1}{n}\sum_{i=1}^{n}\mathbf{z}_i$ with error $\mathcal{O}\left(\frac{d}{n4^{m}}\right)$. For completeness, we prove some properties of this quantization scheme in Appendix~\ref{app:quantization}. 

The private DME mechanism is given in Algorithm~\ref{algo:l_inf}, where  $v$ controls the total privacy of the mechanism. There are two communication parameters: $m$ controls the number of bits for quantization and $s$ controls the number of dimensions used to represent each binary vector. In Theorems~\ref{thm:l_inf_vector_ldp} and ~\ref{thm:l_inf_vector_shuffle}, we present how the privacy and communication parameters $v,m,s$ affects the accuracy of the mechanism. The server-side is presented in Algorithm~\ref{algo:l_inf_server}. The server estimate the mean of each binary vectors $\lbrace b_{i}^{(k)}\rbrace$ and decodes the messages to generate an estimate to true mean $\overline{\mathbf{z}}=\frac{1}{n}\sum_{i=1}^{n}\mathbf{z}_i$. Then, the server scales the vector $\overline{\mathbf{z}}$ to generate an unbiased estimate of the mean $\overline{\mathbf{x}}$.

We prove the bound on the MSE of the proposed mechanisms in the local DP and MMS models in the following theorems, where we defer the proofs to Appendix~\ref{app:l_inf_vector}. For ease of presentation, we provide the order of the achievable MSE and give the $\epsilon_0$-LDP and/or central $\left(\epsilon,\delta\right)$-DP guarantees of our mechanism for both local DP and shuffle models. We track the constants in the MSE in the detailed proofs in Appendix~\ref{app:l_inf_vector}, see \eqref{eq:ExpConst-MSE-linf-LDP}, \eqref{eq:ExpConst-MSE-lind-MMS}. Furthermore, we present RDP guarantees of our mechanisms for both local DP and MMS models in the detailed proofs. We give the outline of the proofs in Section~\ref{sec:Pf_Outlines}. 
 
\begin{theorem}[Local DP model]~\label{thm:l_inf_vector_ldp} The output of the local mechanism $\mathcal{R}^{\ell_{\infty}}_{v,m,s}$ can be represented using $ms\left(\log\left(\lceil d/s\rceil\right)+1\right)$ bits. By choosing $v=\epsilon_0$, the mechanism $\mathcal{R}^{\ell_{\infty}}_{v,m,s}$ satisfies $\epsilon_0$-LDP. Let $\hat{\mathbf{x}}$ be the output of the analyzer $\mathcal{A}^{\ell_{\infty}}$. The estimator $\hat{\mathbf{x}}$ is an unbiased estimate of $\overline{\mathbf{x}}=\frac{1}{n}\sum_{i=1}^{n}\mathbf{x}_i$ with bounded MSE:
\begin{equation}~\label{eqn:mse_l_inf_vector_ldp}
\begin{aligned}
\mathsf{MSE}^{\ell_{\infty}}_{\text{LDP}}&=\sup_{\lbrace \mathbf{x}_i\in\mathbb{B}_{\infty}^{d}\left(r_{\infty}\right)\rbrace}\mathbb{E}\left[\|\hat{\mathbf{x}}-\overline{\mathbf{x}}\|_2^2\right] \\
&= \mathcal{O}\left(\frac{r_{\infty}^{2}d^2}{n}\max\left\{\frac{1}{d4^{m}},\frac{1}{s},\frac{s}{\epsilon_0^2}\right\}\right).
\end{aligned}
\end{equation}
\end{theorem}
Theorem~\ref{thm:l_inf_vector_ldp} shows that each client needs to set $m=1$ and $s=\lceil \epsilon_0 \rceil$ communication bits to achieve MSE $\mathcal{O}\left(\frac{d^2}{n\min\lbrace \epsilon_0,\epsilon_0^2\rbrace}\right)$ when $\epsilon_0\leq \sqrt{d}$. Now, we move to the MMS privacy model.

\begin{theorem}[MMS model]~\label{thm:l_inf_vector_shuffle} The output of the local mechanism $\mathcal{R}^{\ell_{\infty}}_{v,m,s}$ can be represented using $ms\left(\log\left(\lceil d/s\rceil\right)+1\right)$ bits. For every $n\in\mathbb{N}$, $\epsilon\leq 1$, and $\delta\in (0,1)$, the shuffling the outputs of $n$ mechanisms $\mathcal{R}^{\ell_{\infty}}_{v,m,s}$ satisfies $\left(\epsilon,\delta\right)$-DP by choosing $v^2=\frac{sn\epsilon^2}{4\log(1/\delta)}$. Let $\hat{\mathbf{x}}$ be the output of the analyzer $\mathcal{A}^{\ell_{\infty}}$. The estimator $\hat{\mathbf{x}}$ is an unbiased estimate of $\overline{\mathbf{x}}=\frac{1}{n}\sum_{i=1}^{n}\mathbf{x}_i$ with bounded MSE:
\begin{equation}~\label{eqn:mse_l_inf_vector_shuffle}
\begin{aligned}
&\mathsf{MSE}^{\ell_{\infty}}_{\text{MMS}}=\sup_{\lbrace \mathbf{x}_i\in\mathbb{B}_{\infty}^{d}\left(r_{\infty}\right)\rbrace}\mathbb{E}\left[\|\hat{\mathbf{x}}-\overline{\mathbf{x}}\|_2^2\right]\\
&\ = \mathcal{O}\left(\frac{r_{\infty}^{2}d^2}{n^2}\max\left\{\frac{n}{d4^{m}},n\left(\frac{1}{s}-\frac{1}{d}\right),\frac{\log\left(1/\delta\right)}{\epsilon^2}\right\}\right).
\end{aligned}
\end{equation}
\end{theorem}
Theorem~\ref{thm:l_inf_vector_shuffle} shows that each client requires to set $m=\lceil \log\left(n\epsilon^2/d\right)\rceil$ and $s=\mathcal{O}\left(\min\lbrace n\epsilon^2,d\rbrace\right)$ so that the error is bounded by $\mathcal{O}\left(\frac{d^2}{n^2\epsilon^2}\right)$ that matches the MSE of central differential privacy mechanisms.

\begin{remark}[Scalar case]~\label{rem:scalar_MMS} When $d=1$, i.e., scalar case, our MMS algorithm achieves the central DP error $\mathcal{O}\left(\frac{1}{n^2\epsilon^2}\right)$ using $m=\lceil\log\left(n\epsilon^2\right)\rceil$ bits per user. This result covers the private-communication trade-offs for all privacy regimes $\epsilon\in(0,1)$. For example, for $\epsilon=\frac{1}{\sqrt{n}}$, each client needs only a single bit to achieve the central DP error. On the other hand, IKOS mechanism proposed in~\cite{balle2020private,ghazi2020private} requires $\mathcal{O}\left(\log\left(n\right)\right)$-bits of communication. Even when particular regimes of order-optimality are achieved for MMS,  the communication bound is in expectation~\cite{ghazi2021differentially},, whereas ours is deterministic. 
\end{remark}

\subsection{Bounded $\ell_2$-norm Vectors}
\label{sec:l_2_norm}

For private DME $\sum_{i=1}^n\mathbf{x}_i$ where
$\|\mathbf{x}_i\|_{2}\leq r_{2}$ for $i\in[n]$, \emph{i.e.,}
$\ell_2$-bounded, we use the random rotation proposed
in~\cite{suresh2017distributed} to bound the $\ell_{\infty}$-norm of
the vector with radius
$r_{\infty}=\mathcal{O}\left(\frac{r_2}{\sqrt{d}}\right)$ and then we
apply the bounded $\ell_{\infty}$-norm algorithm in
Section~\ref{sec:l_inf_norm}.

\begin{algorithm}[t]
\caption{: Local Randomizer $\mathcal{R}^{\ell_{2}}_{v,m,s}$ }\label{algo:l_2}
\begin{algorithmic}[1]
\State \textbf{Public parameter:} Privacy budget $v$, communication levels $m$, communication coordinates per level $s$, and confidence term $\beta$.
\State \textbf{Input:} $\mathbf{x}_i\in \mathbb{B}^{d}_{2}\left(r_{2}\right)$.
\State Let $U=\frac{1}{\sqrt{d}}\mathbf{H}D$, where $\mathbf{H}$ denotes a Hadamard matrix and $D$ is a diagonal matrix with i.i.d. uniformly random $\lbrace \pm 1\rbrace$ entries.
\State $\mathbf{w}_i\gets W\mathbf{x}_i$
\State $r_{\infty}\gets 10r_2\sqrt{\frac{\log\left(dn/\beta\right)}{d}}$
\For{$j=1,\ldots,d$}
\State $\mathbf{w}_i[j]=\min\left\{r_{\infty},\max\left\{\mathbf{w}_i(j),-r_{\infty}\right\}\right\}$
\EndFor
\State $\mathcal{Y}_i\gets \mathcal{R}^{\ell_{\infty}}_{v,m,s}(\mathbf{w}_i)$
\State \textbf{Return:} The client sends $\mathcal{Y}_i$.
\end{algorithmic}
\end{algorithm}
   
\begin{algorithm}[t]
\caption{: Analyzer $\mathcal{A}^{\ell_{2}}$ }\label{algo:l_2_server}
\begin{algorithmic}[1]
\State \textbf{Inputs:} $\mathcal{Y}_1,\ldots,\mathcal{Y}_n$, where $\mathcal{Y}_i=\left\{\mathcal{Y}_i^{(1)},\ldots,\mathcal{Y}_i^{(m)}\right\}$ is a set of $m$ sets.
\State $\hat{\mathbf{w}}\gets \mathcal{A}^{\ell_{\infty}}\left(\mathcal{Y}_1,\ldots,\mathcal{Y}_n\right)$
\State \textbf{Return:} The server returns $\hat{\mathbf{x}}\gets U^{-1}\hat{\mathbf{w}}$.
\end{algorithmic}
\end{algorithm}

\begin{theorem}[Local DP model]~\label{thm:l_2_vector_ldp} The output of the local mechanism $\mathcal{R}^{\ell_{2}}_{v,m,s}$ can be represented using $sm\left(\log\left(\lceil d/s\rceil\right)+1\right)$ bits. By choosing $v=\epsilon_0$, the mechanism $\mathcal{R}^{\ell_{2}}_{v,m,s}$ satisfies $\epsilon_0$-LDP. Let $\hat{\mathbf{x}}$ be the output of the analyzer $\mathcal{A}^{\ell_{2}}$. With probability at least $1-\beta$, the estimator $\hat{\mathbf{x}}$ is an unbiased estimate of $\overline{\mathbf{x}}=\frac{1}{n}\sum_{i=1}^{n}\mathbf{x}_i$ with  MSE:
\begin{equation}~\label{eqn:mse_l_2_vector_ldp}
\begin{aligned}
&\mathsf{MSE}^{\ell_{2}}_{\text{LDP}}=\sup_{\lbrace \mathbf{x}_i\in\mathbb{B}_{2}^{d}\left(r_{2}\right)\rbrace}\mathbb{E}\left[\|\hat{\mathbf{x}}-\overline{\mathbf{x}}\|_2^2\right]\\
&\quad = \mathcal{O}\left(\frac{r_{2}^{2}\log\left(dn/\beta\right)}{n}\max\left\{\frac{1}{4^{m}},\frac{d}{s},\frac{ds}{\epsilon_0^2}\right\}\right).
\end{aligned}
\end{equation}
\end{theorem}

\begin{theorem}[MMS model]~\label{thm:l_2_vector_shuffle} The output of the local mechanism $\mathcal{R}^{\ell_{2}}_{v,m,s}$ can be represented using $sm\left(\log\left(\lceil d/s\rceil\right)+1\right)$ bits. For every $n\in\mathbb{N}$, $\epsilon\leq 1$, and $\delta\in (0,1)$, the shuffling the outputs of $n$ mechanisms $\mathcal{R}^{\ell_{2}}_{v,m,s}$ satisfies $\left(\epsilon,\delta\right)$-DP by choosing $v^2=\frac{n\epsilon^2}{s\log(1/\delta)}$. Let $\hat{\mathbf{x}}$ be the output of the analyzer $\mathcal{A}^{\ell_{2}}$. With probability at least $1-\beta$ The estimator $\hat{\mathbf{x}}$ is an unbiased estimate of $\overline{\mathbf{x}}=\frac{1}{n}\sum_{i=1}^{n}\mathbf{x}_i$ with MSE:
\begin{small}
\begin{equation}~\label{eqn:mse_l_2_vector_shuffle}
\begin{aligned}
&\mathsf{MSE}^{\ell_{2}}_{\text{MMS}}=\sup_{\lbrace \mathbf{x}_i\in\mathbb{B}_{2}^{d}\left(r_{2}\right)\rbrace}\mathbb{E}\left[\|\hat{\mathbf{x}}-\overline{\mathbf{x}}\|_2^2\right]\\
&\quad = \mathcal{O}\left(\frac{r_{2}^{2}\log\left(dn/\beta\right)}{n^2}\max\left\{\frac{n}{4^{m}},n\left(\frac{d}{s}-1\right),\frac{d\log\left(1/\delta\right)}{\epsilon^2}\right\}\right).
\end{aligned}
\end{equation}
\end{small}
\end{theorem}

\begin{remark}[Kashin's represention]~\label{rem:kashin_MMS} Observe that the MSE in~\eqref{eqn:mse_l_2_vector_shuffle} has a factor of $\left(\log(nd)\right)$ that comes from using the random rotation matrix. We can remove this factor $\log(nd)$ by using the Kashin's representation~\cite{kashin1977diameters} to transform the bounded $\ell_2$-norm vector into a bounded $\ell_{\infty}$-norm vector (see e.g.,~\cite{lyubarskii2010uncertainty,caldas2018expanding,chen2020breaking})  
\end{remark}

\begin{remark}[Comparison with SecAgg]~\label{rem:comp_secagg} When $d<n\epsilon^2$, our MMS algorithm requires $\mathcal{O}\left(d\log\left(\frac{n\epsilon^2}{d}\right)\right)$ bits per client to achieve the central DP error $\mathcal{O}\left(\frac{d}{n^2\epsilon^2}\right)$. Furthermore, it requires only $\mathcal{O}\left(n\epsilon^2\log\left(\frac{d}{n\epsilon^2}\right)\right)$-bits when $d>n\epsilon^2$. On the other hand DDG algorithm~\cite{DDG-ICML21} need $\mathcal{O}\left(d\log\left(n\right)\right)$-bits when $d<n^2\epsilon^2$ and $\mathcal{O}\left(n^2\epsilon^2\log\left(n\right)\right)$-bits when $d>n^2\epsilon^2$~\cite{pmlr-v162-chen22c} to achieve the same MSE. Hence, the MMS saves communication in comparison with SecAgg. 
\end{remark}

\begin{remark}[Compatability with SecAgg]~\label{rem:secAgg} When choosing $s=d$, the output of our algorithm $\mathcal{R}^{\ell_2}_{v,m,s}$ can be represented as $m$ binary-vectors. Hence, it is compatible with secure aggregation to compute the sum of these vectors. Thus, using our $\mathcal{R}^{\ell_2}_{v,m,d}$ with SecAgg gives the same privacy-communication trade-offs as the MMS model in Theorem~\ref{thm:l_2_vector_shuffle}, since SecAgg can be seen as a post-processing of shuffling. However, our algorithm needs $d\lceil\log\left(\frac{n\epsilon^2}{d}\right)\rceil$-bits per client to achieve the central error of $\mathcal{O}\left(\frac{d}{n^2\epsilon^2}\right)$. On the other hand, the distributed-discrete-Gaussian in~\cite{DDG-ICML21} needs $\mathcal{O}\left(d\log\left(n\right)\right)$-bits per client to achieve the same MSE. 
\end{remark}

Next we present a lower bound for DME under privacy and communication constraints, which can be derived using results from \cite{pmlr-v162-chen22c} and ~\cite{bun2014fingerprinting}. 
  
\begin{theorem}[Lower Bound For central DP model]~\label{thm:L_2_lower_bound_central} Let $n,d\in\mathbb{N}$, $\epsilon>0$, $r_2\geq1$, and $\delta=o(\frac{1}{n})$. For any $\mathbf{x}_1,\ldots,\mathbf{x}_n\in\mathbb{B}_2^{d}(r_2)$, the MSE is bounded below by:
\begin{equation}
\mathsf{MSE}_{\text{central}}^{\ell_2}=\Omega\left(r_2^2\max\left\{\frac{d}{n^2\epsilon^2},\frac{1}{n4^{b/d}}\right\}\right)
\end{equation}
for any unbiased algorithm $\mathcal{M}$ that is $\left(\epsilon,\delta\right)$-DP with $b>d$-bits of communication per client. Furthermore, when $b<d$ bits per client, the MSE is bounded below by:
 \begin{equation}
\mathsf{MSE}_{\text{central}}^{\ell_2}=\Omega\left(r_2^2d\max\left\{\frac{1}{n^2\epsilon^2},\frac{1}{nb} \right\}\right)
\end{equation}
\end{theorem}

\begin{algorithm}[t]
\caption{: Local Randomizer $\mathcal{R}^{\text{Bin}}_{p,s}$ }\label{algo:binary_vector}
\begin{algorithmic}[1]
\State \textbf{Public parameter:} Privacy parameter $p$, and communication budget $s$.
\State \textbf{Input:} $\mathbf{b}_i\in\lbrace 0,1\rbrace^{d}$.
\State $a\gets \lceil \frac{d}{s}\rceil$
\State If $a$ is not integer, add $(sa-d)$ dummy zeros to the binary vector $\mathbf{b}$.
\For{$j\in [s]$}
\State Choose uniformly at random one coordinate $a_{ij}\gets\mathsf{Unif}\left(\lbrace (j-1)a,\ldots,ja\rbrace\right)$.
\State $y_{ij}\gets a\mathcal{R}^{\textsl{2RR}}_p\left(\mathbf{b}_i[a_{ij}]\right)$
\EndFor
\State \textbf{Return:} The client sends $s$ messages $\mathcal{Y}_i\gets\left\{\left(z_{i1},y_{i1}\right),\ldots,\left(z_{is},y_{is}\right)\right\}$.
\end{algorithmic}
\end{algorithm}

\begin{remark}(Optimality of our mechanism)\label{rem:MMS-optimal} When the communication budget $b>d$, we can see that our MSE in Theorem~\ref{thm:l_2_vector_shuffle} matches the lower bound in~\ref{thm:L_2_lower_bound_central} (up to logarithmic factor) by choosing $s=d$ and $m=b/d$. Furthermore, when the communication budget $b<d$, our algorithm achieve the lower bound by choosing $s=b$ and $m=1$. Thus, our algorithm for MMS is order optimal for all privacy-communication regimes. 
\end{remark}

%% file: Proof_outlines.tex
\section{Proof outlines}~\label{sec:Pf_Outlines}
As can be seen from \eqref{eq:BinExp}, and Algorithm~\ref{algo:l_inf},
the main ingredient is to solve the following sub-problem. Suppose,
each client has a binary vector $\mathbf{b}_i\in\lbrace
0,1\rbrace^{d}$. The goal is to privately compute the sum
$\overline{\mathbf{b}}=\frac{1}{n}\sum_{i=1}^{n}\mathbf{b}_i$ under
privacy and communication constraints. If we can demonstrate a
solution to this problem, then we can apply it to bounded-norm vectors
as in Sections \ref{sec:l_inf_norm} and \ref{sec:l_2_norm} using
\eqref{eq:BinExp}, along with another critical ingredient, to
judiciously allocate the overall privacy budget among these
bit-vectors describing the vectors at different resolution. These are
the two main ideas that enable us to get the main theoretical results.

\begin{algorithm}[t]
\caption{: Analyzer $\mathcal{A}^{\text{Bin}}$ }\label{algo:binary_vector_server}
\begin{algorithmic}[1]
\State \textbf{Inputs:} $\mathcal{Y}_1,\ldots,\mathcal{Y}_n$, where $\mathcal{Y}_i$ is $s$ messages each is a pair $(a_{ij},y_{ij})$ for $j\in[s]$ and $i\in[n]$. 
\State $\hat{\mathbf{b}}\gets \mathbf{0}_{d}$
\For{$i\in[n]$}
\For{$j\in [s]$}
\State $\hat{\mathbf{b}}[a_{ij}]\gets \hat{\mathbf{b}}[a_{ij}]+y_{ij}$.
\EndFor
\EndFor
\State $\hat{\mathbf{b}}\gets \frac{1}{n}\hat{\mathbf{b}}$
\State \textbf{Return:} The server returns $\hat{\mathbf{b}}$.
\end{algorithmic}
\end{algorithm}
\subsection{Binary vectors}~\label{sec:binary_vector}
A straightforward solution to compute $\overline{\mathbf{b}}=\frac{1}{n}\sum_{i=1}^{n}\mathbf{b}_i$, is to apply the scalar solution proposed in ~\cite{cheu2019distributed} for each coordinate.  However, this requires $d$ bits per client. We will design private mechanisms with much less communication budget per client.

The client-side mechanism is presented in Algorithm~\ref{algo:binary_vector}, where the parameter $s$ determines the communication budget for each client and the parameter $p$ determines the total privacy budget (see Theorem~\ref{thm:binary_vector_ldp}). For given $s\in\lbrace 1,\ldots,d\rbrace$, each client splits the binary vector $\mathbf{b}_i$ into $s$ sub-vectors each with dimension $a = \lceil \frac{d}{s}\rceil$. Then, the client chooses uniformly at random one coordinate from each sub-vector and privatizes its bit using 2RR Algorithm~\ref{algo:2RR}. Observe that the output of Algorithm~\ref{algo:binary_vector} can be represented as a sparse $d$-dimensional vector with only $s$ non-zero bits.

When $s=1$, then each client applies the 2RR mechanism on each coordinate separately. On the other hand, when $s=d$, the client chooses uniformly at random one coordinate and applies the 2RR mechanism. Thus, we get trade-offs between privacy-communication and accuracy. The server aggregator $\mathcal{A}^{\text{Bin}}$ is presented in Algorithm~\ref{algo:binary_vector_server}, where the server simply aggregates the received randomized bits. 

In the following theorems, we prove the bound on the MSE of the proposed mechanisms in the local DP and shuffle models. The proofs are deferred to Appendix~\ref{app:binary_vector}. For ease of presentation, we provide the order of the achievable MSE and give the $\epsilon_0$-LDP and/or central $\left(\epsilon,\delta\right)$-DP guarantees of our mechanism for both local DP and shuffle models. However, we track the constants in the MSE in the detailed proofs in Appendix~\ref{app:binary_vector}. Furthermore, we present RDP guarantees of our mechanisms for both local DP and shuffle models in the detailed proofs.   
 
\begin{theorem}[Local DP model]~\label{thm:binary_vector_ldp} The output of the local mechanism $\mathcal{R}^{\text{Bin}}_{p,s}$ can be represented using $s\left(\log\left(\lceil d/s\rceil\right)+1\right)$ bits. By choosing $p=\frac{1}{2}\left(1-\sqrt{\frac{\epsilon_0^2/s^2}{\epsilon_0^2/s^2+4}}\right)$, the mechanism $\mathcal{R}_{p,s}^{\text{Bin}}$ satisfies $\epsilon_0$-LDP. Let $\hat{\mathbf{b}}$ be the output of the analyzer $\mathcal{A}^{\text{Bin}}$. The estimator $\hat{\mathbf{b}}$ is an unbiased estimate of $\overline{\mathbf{b}}=\frac{1}{n}\sum_{i=1}^{n}\mathbf{b}_i$ with bounded MSE:
\begin{equation}~\label{eqn:mse_binary_vector_ldp}
\begin{aligned}
\mathsf{MSE}^{\text{Bin}}_{\text{ldp}}&=\sup_{\lbrace \mathbf{b}_i\in \lbrace0,1\rbrace^{d}\rbrace}\mathbb{E}\left[\|\hat{\mathbf{b}}-\overline{\mathbf{b}}\|_2^2\right]\\
& = \mathcal{O}\left(\frac{d^2}{n}\max\left\{\frac{1}{s},\frac{s}{\epsilon_0^2}\right\}\right).
\end{aligned}
\end{equation}
\end{theorem}
Theorem~\ref{thm:binary_vector_ldp} shows that each client needs to send $s=\lceil \epsilon_0 \rceil$ communication bits to achieve MSE $\mathcal{O}\left(\frac{d^2}{n\min\lbrace \epsilon_0,\epsilon_0^2\rbrace}\right)$. Now, we move to the shuffle model, where we assume that there exists $s$ shuffler. The $j$-th shuffler randomly permutes the set of messages $\left\{\left(a_{ij},y_{ij}\right):i\in[n]\right\}$ from the $n$ clients.

\begin{theorem}[MMS model]~\label{thm:binary_vector_shuffle} The output of the local mechanism $\mathcal{R}^{\text{Bin}}_{p,s}$ can be represented using $s\left(\log\left(\lceil d/s\rceil\right)+1\right)$ bits. For every $n\in\mathbb{N}$, $\epsilon\leq 1$, and $\delta\in (0,1)$, shuffling the outputs of $n$ mechanisms $\mathcal{R}_{p,s}^{\text{Bin}}$ satisfies $\left(\epsilon,\delta\right)$-DP by choosing $p=\frac{1}{2}\left(1-\sqrt{\frac{v^2}{v^2+4}}\right)$, where $v^2=\frac{n\epsilon^2}{4s\log(1/\delta)}$. Let $\hat{\mathbf{b}}$ be the output of the analyzer $\mathcal{A}^{\text{Bin}}$. The estimator $\hat{\mathbf{b}}$ is an unbiased estimate of $\overline{\mathbf{b}}=\frac{1}{n}\sum_{i=1}^{n}\mathbf{b}_i$ with bounded MSE:
\begin{equation}~\label{eqn:mse_binary_vector_shuffle}
\begin{aligned}
\mathsf{MSE}^{\text{Bin}}_{\text{shuffle}}&=\sup_{\lbrace \mathbf{b}_i\in \lbrace0,1\rbrace^{d}\rbrace}\mathbb{E}\left[\|\hat{\mathbf{b}}-\overline{\mathbf{b}}\|_2^2\right]\\
& = \mathcal{O}\left(\frac{d^2}{n^2}\max\left\{n\left(\frac{1}{s}-\frac{1}{d}\right),\frac{\log\left(1/\delta\right)}{\epsilon^2}\right\}\right).
\end{aligned}
\end{equation}
\end{theorem}
Theorem~\ref{thm:binary_vector_shuffle} shows that each client requires to send $s=\mathcal{O}\left(\min\lbrace n\epsilon^2,d\rbrace\right)$ communication bits such that the error in the shuffle model is bounded by $\mathcal{O}\left(\frac{d^2}{n^2\epsilon^2}\right)$ that matches the MSE of central differential privacy mechanisms. For the scalar case when $d=1$, our results in Theorem~\ref{thm:binary_vector_shuffle} matches the optimal MSE as in~\cite{cheu2019distributed}. 
\subsection{Putting things together} \label{subsec:Combine}

We start with proof outlines for Theorems~\ref{thm:l_inf_vector_ldp} and ~\ref{thm:l_inf_vector_shuffle}. For both, the local randomization is the same, and the basic idea is that of non-uniform randomization of the different bits used to quantize a real vector $\mathbf{z}_i$, arising from \eqref{eq:BinExp}. In particular, we use distinct randomizations for each bit vector $\mathbf{b}_i^{(k)}\in\lbrace 0,1\rbrace^{d}$, with different parameters $p_i$ causing different privacy for each resolution level $k$. For a given local privacy guarantee of $\epsilon_0$, we divide this into guarantees $\epsilon_0^{(k)}$ for the $k$-th resolution level, such that $\epsilon_0=\sum_{k=1}^{m}\epsilon_0^{(k)}$. The intuition is that one allocates higher privacy (lower $\epsilon_0^{(k)}$) to the MSBs (lower $k$), for a given overall privacy budget $\epsilon_0$. This is because to get better accuracy (performance in terms of MSE) we want the higher-order bits to be less noisy than the lower-order bits.  We connect this non-uniform choice to the MSE for the LDP and MMS privacy models below.

\begin{figure*}[t]
    \centering 
\begin{subfigure}[b]{0.31\textwidth}
  \includegraphics[scale=0.32]{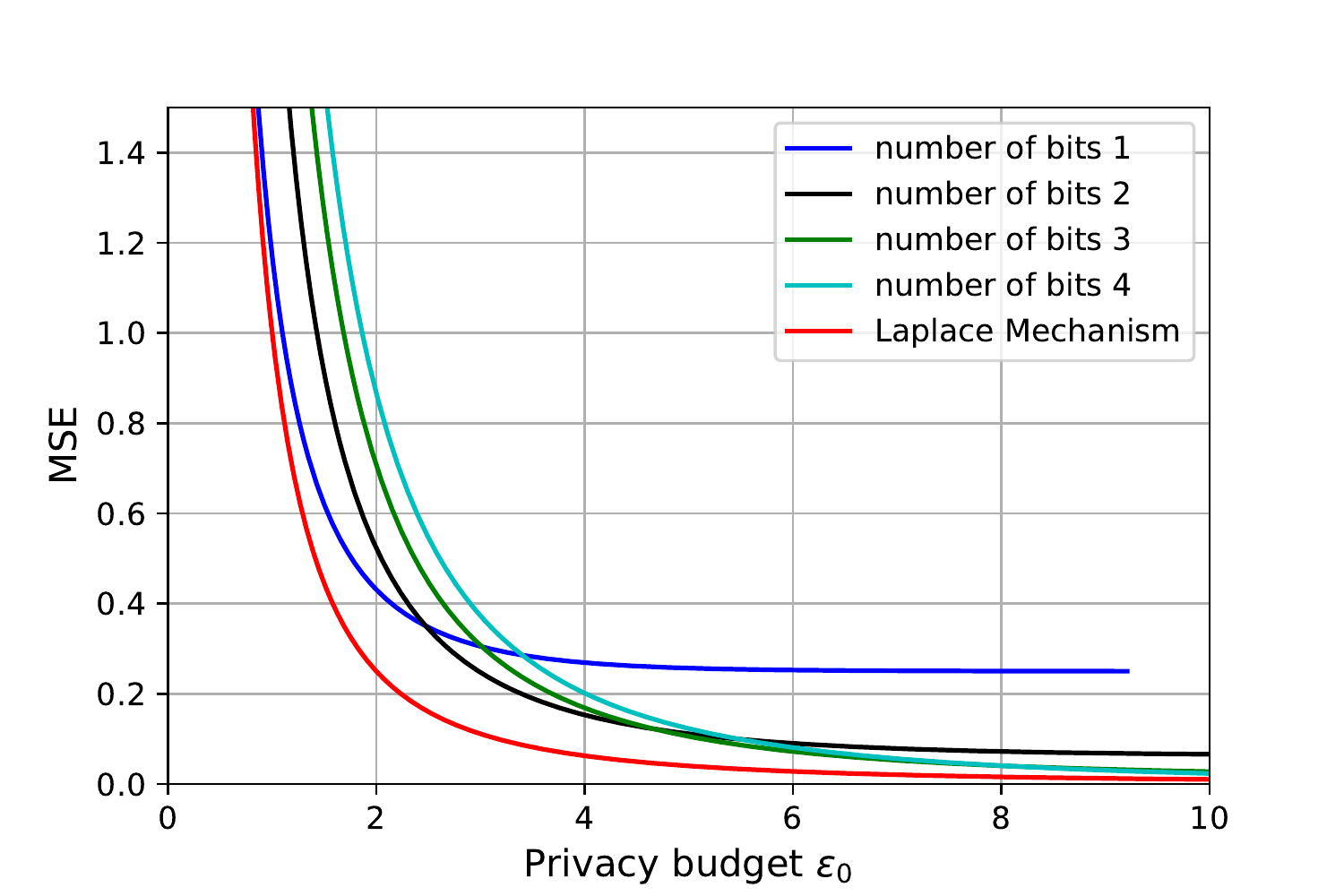}
\caption{Comparison of our LDP mechanism $\mathcal{R}^{\ell{\infty}}_{v,m,s}$ with Laplace mechanism for $d=1$, $n=1$, and $m\in\lbrace 1,2,3,4\rbrace$.} 
\label{fig:local}
\end{subfigure}\hfil 
\begin{subfigure}[b]{0.31\textwidth}
  \includegraphics[scale=0.32]{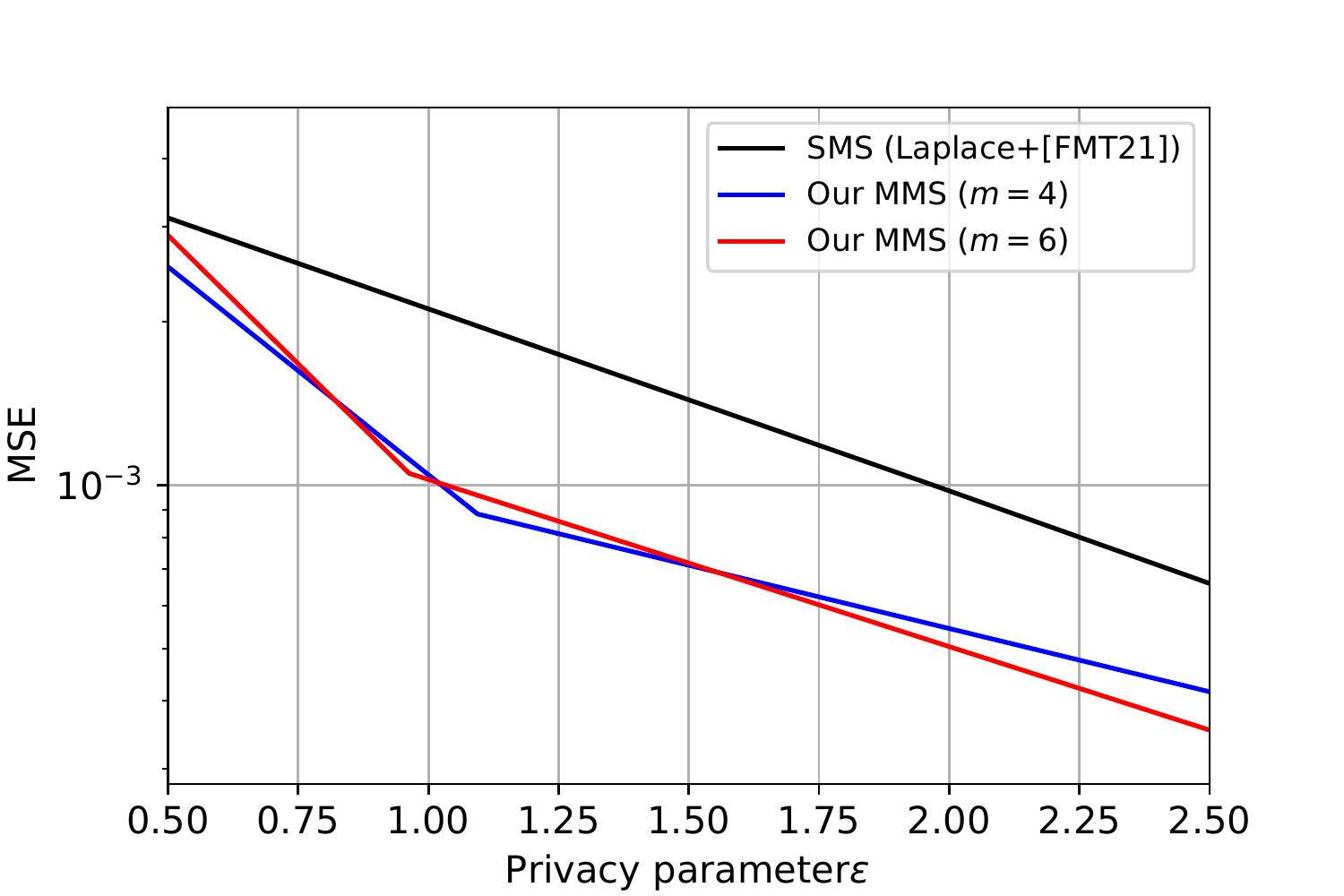}
\caption{Comparison of our MMS mechanism $\mathcal{R}^{\ell{\infty}}_{v,m,s}$ with SMS (Laplace+[FMT21]) for $d=1$, $n=1000$, and $m\in\lbrace 4,6\rbrace$.} 
\label{fig:shuffle_scalar}
\end{subfigure}\hfil 
\begin{subfigure}[b]{0.31\textwidth}
  \includegraphics[scale=0.32]{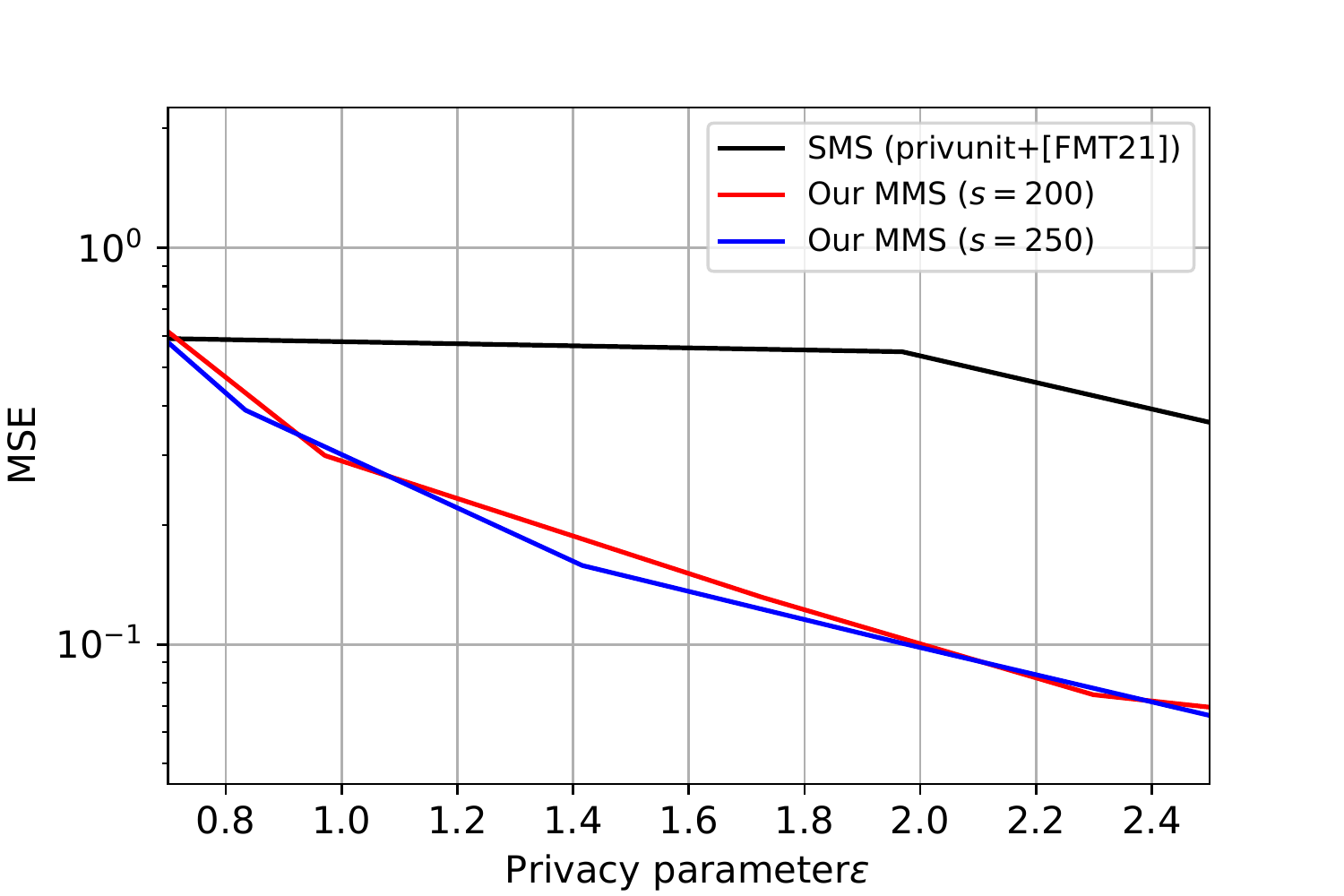}
\caption{Comparison of our MMS mechanism $\mathcal{R}^{\ell{2}}_{v,m,s}$ with SMS (privunit+[FMT21]) for $d=300$, $n=1000$, and $s\in\lbrace 200,250\rbrace$.} 
\label{fig:shuffle_vector}
\end{subfigure}
\end{figure*}

\begin{lemma}[Non-uniform privacy allocation]
  \label{lem:Priv-alloc}
  Consider $m$ privacy mechanisms for $\{\mathbf{b}_i^{(k)}\in\lbrace 0,1\rbrace^{d}\},\mathbf{u}_i$ denoted by $\mathcal{R}^{\text{Bin}}_{p_1,s}(\mathbf{b}_i^{(1)}),\ldots,\mathcal{R}^{\text{Bin}}_{p_{m-1},s}(\mathbf{b}_i^{(m-1)}),\mathcal{R}^{\text{Bin}}_{p_m,s}(\mathbf{u}_i)$, parametrized by $\{p_i\}$. For a given total privacy allocation of the choice of $v\stackrel{\triangle}{=}\epsilon_0$, the choice of $v_k\stackrel{\triangle}{=}\epsilon_0^{(k)}=\frac{4^{\frac{-k}{3}}}{\left(\sum_{l=1}^{m-1}4^{\frac{-l}{3}}+4^{\frac{-m+1}{3}}\right)}v$ for $k\in[m-1]$ and $v_m=\frac{4^{\frac{-m+1}{3}}}{\left(\sum_{l=1}^{m-1}4^{\frac{-l}{3}}+4^{\frac{-m+1}{3}}\right)}v$, we can get the following LDP and MMS models' RDP-privacy guarantees:
  \begin{align}
    \epsilon_{\text{LDP}}\left(\alpha\right)&=\sum_{k=1}^{m}\epsilon_{\text{LDP}}^{(k)}\left(\alpha\right) \\ \label{eq:RDP-LDP} 
    \epsilon_{\text{MMS}}\left(\alpha\right)&=\sum_{k=1}^{m}\epsilon_{\text{MMS}}^{(k)}\left(\alpha\right) \leq c\frac{\alpha v^2}{sn}
  \end{align}
for some constant $c$ and $\epsilon_{\text{LDP}}^{(k)}\left(\alpha\right)\leq\frac{s}{\alpha-1}\log\Big(p^{\alpha}_k$ $(1-p_k)^{1-\alpha}+p^{1-\alpha}_k(1-p_k)^{\alpha}\Big)$ (see Appendix~\ref{app:l_inf_vector} for details).
\end{lemma}
This lemma immediately yields the central DP guarantees of $\epsilon_0$ for the LDP model, and a $\left(\epsilon_{\text{MMS}},\delta\right)$-DP, for the MMS model, where $\epsilon_{\text{MMS}}$ is bounded by
\begin{equation}~\label{eqn:eps_delta_l_inf}
\epsilon_{\text{MMS}}\leq 2c\sqrt{\frac{\epsilon_0^2\log(1/\delta)}{sn}},
\end{equation}
which suggests setting $\epsilon_0^2=\frac{sn\epsilon^2}{4\log(1/\delta)}$, for the local randomization. Critically, this choice of non-uniform privatization enables the following result, proved in Appendix~\ref{app:l_inf_vector}.

\begin{lemma}[MSE performance]
  \label{lem:MSE-perf}
  With the non-uniform privacy allocation specified in Lemma~\ref{lem:Priv-alloc}, we get the the following LDP and MMS models' MSE performance for DME:
\begin{small}
\begin{align}
  \label{eq:MSE-LDP}
\mathsf{MSE}^{\ell_{\infty}}_{\text{LDP}}&\leq\mathcal{O}\left(\frac{r_{\infty}^2d^2}{n}\max\left\{\frac{1}{d4^{m}},\frac{1}{s},\frac{s}{\epsilon_0^2}\right\}\right)  
\\ \label{eq:MSE-MMS}
\mathsf{MSE}^{\ell_{\infty}}_{\text{MMS}}  &\leq\mathcal{O}\left(\frac{r_{\infty}^2d^2}{n^2}\max\left\{\frac{n}{d4^{m}},n\left(\frac{1}{s}-\frac{1}{d}\right),\frac{\log\left(1/\delta\right)}{\epsilon^2}\right\}\right)
\end{align}
\end{small}
\end{lemma}

Theorem~\ref{thm:l_inf_vector_ldp} follows from \eqref{eq:MSE-LDP} and Theorem~~\ref{thm:l_inf_vector_shuffle} follows from \eqref{eqn:eps_delta_l_inf} and \eqref{eq:MSE-MMS}. Theorems~\ref{thm:l_2_vector_ldp}, \ref{thm:l_2_vector_shuffle}  directly follow by using Theorem \ref{thm:bounded_norm_2} in Appendix~\ref{app:l_2_vector} in ~\ref{thm:l_inf_vector_ldp} and ~\ref{thm:l_inf_vector_shuffle}.

%% file: Numerics.tex
\section{Numerical Results}
\label{sec:numerics}
In this section, we evaluate the performance of our algorithms in the local DP model and the shuffle model.

\subsection{Local DP model}
We start by comparing the performance of our algorithm $\mathcal{R}^{\ell{\infty}}_{v,m,s}$ with the performance of the Laplace mechanism~\cite{pmlr-v162-chen22c} in the local model for scalar case, i.e., $d=1$. Hence, the elements $\mathbf{x}_i\in[-1,1]$. Observe that the Laplace mechanism is the optimal scheme is this case, however, it has infinite communication bits. In Figure~\ref{fig:local}, we plot the MSE of our $\mathcal{R}^{\ell{\infty}}_{v,m,s}$ with different communication budget $s=1$ and $m\in\lbrace1,2,3,4\rbrace$ for a single client $n=1$. We can observe that our mechanism achieves MSE closer to the MSE of the Laplace mechanism. Furthermore, we only need at most $m=3$ bits to achieve similar performance as Laplace mechanism.  
\subsection{Shuffler model}
We consider two cases in the shuffler model: 1) The scalar case when $d=1$ to evaluate the performance of our $\mathcal{R}^{\ell{\infty}}_{v,m,s}$ mechanism in the shuffle model. 2) The vector case when $d=1000$ to evaluate the performance of our $\mathcal{R}^{\ell{2}}_{v,m,s}$ mechanism in the shuffle model.\\
\paragraph{Scalar} In Figure~\ref{fig:shuffle_scalar}, we plot the MSE of two different mechanisms versus the central privacy $\epsilon$ for fixed $\delta=10^{-5}$. The first mechanism is single message shuffle (SMS) obtained using Laplace mechanism with privacy amplification results in~\cite{}. Observe that Laplace mechanism is the optimal LDP mechanism for LDP and the privacy amplification results in~\cite{feldman2022hiding} is approximately optimal for $\left(\epsilon,\delta\right)$-DP. Hence, we expect that this is the best that an SMS mechanism can achieve. The second mechanism is our multi-message shuffling (MMS) mechanism $\mathcal{R}^{\ell{\infty}}_{v,m,s}$ mechanism for $d=1$ and $m\in\lbrace 4,6\rbrace$. Since we have MMS, we use the RDP results of privacy amplification by shuffling in~\cite{girgis2021renyi-CCS} which is better for composition to compute the RDP of our mechanism. Then, we transform from RDP bound to approximate $\left(\epsilon,\delta\right)$-DP. We choose number of clients $n=1000$. We can see that our multi-message shuffle model achieve lower MSE than the single message shuffle especially for large value of central DP parameter $\epsilon$.

\paragraph{Bounded $\ell_2$-norm vectors} Similar to the scalar case, we consider two mechanisms. The first mechanism SMS is obtained by using \texttt{privunit} mechanism with the privacy amplification results in~\cite{feldman2022hiding}, where \texttt{privunit}~\cite{bhowmick2018protection} is asymptotically optimal LDP mechanism~\cite{asi2022optimal}. We choose $n=1000$ and $d=300$. For our MMS $\mathcal{R}^{\ell_2}_{v,m,s}$, we choose $s\in\lbrace 200,250\rbrace$. It is clear from Figure~\ref{fig:shuffle_vector} that our MMS mechanism has better performance than SMS mechanism.

%% file: App_2RR.tex
\section{Binary Randomized Response}~\label{app:2rr_mse}

In this section we review an unbiased version of the classical binary
randomized response (\textsl{2RR} mechanism) in
Algorithm~\ref{algo:2RR}. We also gather some results on the classical
binary randomized response, which will be useful for our proofs.

\begin{theorem}[Repeating Theorem~\ref{thm:binary_mse}]~\label{thm_app:binary_mse} For any $p\in[0,1/2)$, the \textsl{2RR} is $\epsilon_0$-LDP, where $\epsilon_0=\log\left(\frac{1-p}{p}\right)$. The output $y$ of the \textsl{2RR} mechanism is an unbiased estimate of $b$ with bounded MSE:
\begin{equation}~\label{eqn:app_mse_binary_sum}
\mathsf{MSE}^{\textsl{2RR}}=\sup_{ b\rbrace\in\lbrace 0,1}\mathbb{E}\left[\|b-y\|_2^2\right] = \frac{p(1-p)}{(1-2p)^2}.
\end{equation}
\end{theorem}

\noindent\textbf{Proof of Theorem~\ref{thm:binary_mse}} (The MSE of the \textsl{2RR})
First, we show that the output of Algorithm~\ref{algo:2RR} is unbiased estimate of $b$. Let $y$ be the output of the 2RR Algorithm~\ref{algo:2RR}. Then, we have 
\begin{equation}
\begin{aligned}
\mathbb{E}\left[y\right]&=\frac{b-p}{1-2p} (1-p)+\frac{1-b-p}{1-2p} p\\
&=b\left(\frac{1-2p}{1-2p}\right)-\frac{p(1-p)}{1-2p}+\frac{p(1-p)}{1-2p}\\
&= b.
\end{aligned}
\end{equation}
Hence, the Algorithm~\ref{algo:2RR} is an unbiased estimate of the input $b$. Furthermore, the MSE of the 2RR is bounded by:
\begin{equation}
\begin{aligned}
\mathsf{MSE}^{\textsl{2RR}}&=
\mathbb{E}\left[\|y-b\|^2\right]=\mathbb{E}\left[y^2\right]-b^2\\
&=\frac{1}{(1-2p)^2}\left[(b-p)^2(1-p)+(1-b-p)^2p\right]-b^2\\
&=\frac{1}{(1-2p)^2}\left[b^2-4p(1-p)b+p(1-p)\right]-b^2\\
&=\frac{1}{(1-2p)^2}\left[b^2-4p(1-p)b+p(1-p)\right]-b^2\\
&=\frac{1}{(1-2p)^2}\left[b^2(4p(1-p))-4p(1-p)b+p(1-p)\right]\\
&=\frac{p(1-p)}{(1-2p)^2}.
\end{aligned}
\end{equation}
The LDP guarantees of the \textsl{2RR} is obtained from the fact that $e^{-\epsilon_0}\leq 1\leq\frac{1-p}{p}\leq e^{\epsilon_0}$ for any $p\in(0,1/2]$.
Furthermore, we can prove that the \textsl{2RR} satisfies $\left(\alpha,\epsilon(\alpha)\right)$-RDP, where $\epsilon\left(\alpha\right)$ is given by:
\begin{equation}
\epsilon\left(\alpha\right)=\frac{1}{\alpha-1}\log\left(p^{\alpha}(1-p)^{1-\alpha}+p^{1-\alpha}(1-p)^{\alpha}\right),
\end{equation}
where this bound is obtained from the definition of the RDP and also given in~\cite{mironov2017renyi}. This completes the proof of Theorem~\ref{thm:binary_mse}.
\hfill $\blacksquare$

Next we present the following lemma which is useful for bounding the privacy parameter, $\epsilon_0$, parameter of our mechanisms which depend on the binary randomized response.
\begin{lemma}(Privacy parameter)~\label{lemm:eps_0_LDP} For any $v>0$, by setting $p=\frac{1}{2}\left(1-\sqrt{\frac{v^2}{v^2+4}}\right)$, the \textsl{2RR} mechanism with parameter $p$ satisfies $\epsilon_0$-LDP, where $\epsilon_0\leq v$.
\end{lemma} 

\begin{proof}
From Theorem~\ref{thm:binary_mse}, the \textsl{2RR} mechanism with parameter $p<1/2$ is $\epsilon_0$-LDP, where $\epsilon_0=\log\left(\frac{1-p}{p}\right)$. Hence, it is sufficient to prove that $\epsilon_0=\log\left(\frac{1-p}{p}\right)\leq v$ when choosing $p=\frac{1}{2}\left(1-\sqrt{\frac{v^2}{v^2+4}}\right)$ for any $v\geq 0$.

Observe that $1-p=\frac{1}{2}\left(1+\sqrt{\frac{v^2}{v^2+4}}\right)$ when $p=\frac{1}{2}\left(1-\sqrt{\frac{v^2}{v^2+4}}\right)$. Let $f(v)=v-\log\left(\frac{\sqrt{v^2+4}+v}{\sqrt{v^2+4}-v}\right)$. We have that
\begin{equation}
\begin{aligned}
\frac{\partial f}{\partial v} &= 1 -\frac{\sqrt{v^2+4}-v}{\sqrt{v^2+4}+v}\frac{8}{\left(\sqrt{v^2+4}-v\right)^2\sqrt{v^2+4}} \\
&=1 -\frac{8}{(v^2+4-v^2)\sqrt{v^2+4}}\\
&=1-\frac{2}{\sqrt{v^2+4}}\\
&\geq 0\qquad \forall\ v\geq 0.
\end{aligned}
\end{equation}
Hence the function $f(v)$ is a non-decreasing function for all $v\geq 0$. As a result $f(v)\geq f(0)=0$ for all $v\geq 0$. Thus, we have $v\geq \log\left(\frac{1-p}{p}\right)$ for all $v\geq 0$. This completes the proof of Lemma~\ref{lemm:eps_0_LDP}. 
\end{proof}

%% file: App_binary_vector.tex
\section{Proofs of Theorem~\ref{thm:binary_vector_ldp} and Theorem~\ref{thm:binary_vector_shuffle} (Binary vectors)}~\label{app:binary_vector}

In this section, we prove Theorem~\ref{thm:binary_vector_ldp} and Theorem~\ref{thm:binary_vector_shuffle} for the mean of binary vectors in local DP and MMS models, respectively.
\subsection{Communication Bound for Theorem~\ref{thm:binary_vector_ldp} and Theorem~\ref{thm:binary_vector_shuffle}}
Observe that each client sends $s$ messages; each message consists of a pair $\left(a_{ij},y_{ij}\right)$, where $a_{ij}$ is drawn uniformly at random from $\lceil \frac{d}{s}\rceil$ values and $y_{ij}$ is a binary elements. Hence, each message requires $\log\left(\lceil \frac{d}{s}\rceil\right)+1$ bits. As a result the total communication bits per client is given by $s\left(\log\left(\lceil \frac{d}{s}\rceil\right)+1\right)$-bits.

\subsection{Privacy of the local DP model in Theorem~\ref{thm:binary_vector_ldp}}
In the mechanism $\mathcal{R}^{\text{Bin}}_{p,s}$, each client sends $s$ messages of the \textsl{2RR} mechanism $\left(\left(a_{i1},y_{i1}\right),\ldots,\left(a_{is},y_{is}\right)\right)$ with parameter $p=\frac{1}{2}\left(1-\sqrt{\frac{\epsilon_0^2/s^2}{\epsilon_0^2/s^2+4}}\right)$. Hence, from Lemma~\ref{lemm:eps_0_LDP}, each message is $\frac{\epsilon_0}{s}$-LDP. As a results, the total mechanism $\mathcal{R}^{\text{Bin}}_{p,s}$ is $\epsilon_0$-LDP from the composition of the DP mechanisms~\cite{dwork2014algorithmic}.

In addition, we can bound the RDP of the mechanism $\mathcal{R}^{\text{Bin}}_{p,s}$ in the local DP model by using the composition of the RDP (see Lemma~\ref{lemm:compostion_rdp}). From the proof of Theorem~\ref{thm:binary_mse} in Appendix~\ref{app:2rr_mse}, the \textsl{2RR} mechanism is $\left(\alpha,\epsilon\left(\alpha\right)\right)$-RDP, where $\epsilon\left(\alpha\right)$ is bounded by:
\begin{equation}~\label{eqn:rdp_2rr_bound}
\epsilon\left(\alpha\right)=\frac{1}{\alpha-1}\log\left(p^{\alpha}(1-p)^{1-\alpha}+p^{1-\alpha}(1-p)^{\alpha}\right),
\end{equation}
In the mechanism $\mathcal{R}^{\text{Bin}}_{p,s}$, each client sends $s$ messages of the \textsl{2RR} mechanism. Hence, the mechanism $\mathcal{R}^{\text{Bin}}_{p,s}$ is $\left(\alpha,s\epsilon\left(\alpha\right)\right)$-RDP, where $\epsilon\left(\alpha\right)$ is given is~\eqref{eqn:rdp_2rr_bound}. 

\subsection{Privacy of the MMS model in Theorem~\ref{thm:binary_vector_shuffle}}
In the mechanism $\mathcal{R}^{\text{Bin}}_{p,s}$, each client sends $s$ messages of the \textsl{2RR} mechanism $\left(\left(a_{i1},y_{i1}\right),\ldots,\left(a_{is},y_{is}\right)\right)$. We assume that there exist $s$ shuffler, where the $j$-th shuffler randomly permutes the set of messages $\left\{\left(a_{ij},y_{ij}\right):i\in[n]\right\}$ from the $n$ clients. Hence from composition of the RDP, it is sufficient to bound the RDP of shuffling $n$ outputs of the \textsl{2RR} mechanism.

We use the recent results of privacy amplification by shuffling in~\cite{girgis2021renyi-CCS}, which states the following
\begin{lemma}~\cite{girgis2021renyi-CCS}
\label{lem:RDP-CCS21} For any $n\in\mathbb{N}$, $\epsilon_0>0$, and $\alpha$ such that $\alpha^{4}e^{5\epsilon_0}\leq \frac{n}{9}$, the output of shuffling $n$ messages of an $\epsilon_0$-LDP mechanism is $\left(\alpha,\epsilon\left(\alpha\right)\right)$-RDP, where $\epsilon\left(\alpha\right)$ is bounded by:
\begin{equation}
\epsilon\left(\alpha\right)\leq \frac{1}{\alpha-1}\log\left(1+\alpha(\alpha-1)\frac{2\left(e^{\epsilon_0}-1\right)^2}{n}\right)\leq 2\alpha\frac{\left(e^{\epsilon_0}-1\right)^2}{n}
\end{equation} 
\end{lemma}

Recently~\cite{feldman2022stronger} improved the dependence on
$\epsilon_0$ of the result in ~\cite{girgis2021renyi-CCS} by showing
the following.
\begin{lemma}~\cite{feldman2022stronger}[Corollary 4.3]~\label{lemm:rdp_bound_feldman} For any $n\in\mathbb{N}$, $\epsilon_0>0$, and $\alpha\leq\frac{n}{16\epsilon_0e^{\epsilon_0}}$, the output of shuffling $n$ messages of an $\epsilon_0$-LDP mechanism is $\left(\alpha,\epsilon\left(\alpha\right)\right)$-RDP, where $\epsilon\left(\alpha\right)$ is bounded by:
\begin{equation}
\epsilon\left(\alpha\right)\leq \alpha \frac{c\left(e^{\epsilon_0}-1\right)^2}{ne^{\epsilon_0}},
\end{equation} 
for some universal constant $c$. 
\end{lemma}

From Theorem~\ref{thm:binary_mse}, each message of the client is $\epsilon_0=\log\left(\frac{1-p}{p}\right)$-LDP. Hence, from Lemma~\ref{lemm:rdp_bound_feldman}, the output of one shuffler is $\left(\alpha,\tilde{\epsilon}\left(\alpha\right)\right)$-RDP, where $\tilde{\epsilon}\left(\alpha\right)\leq c\alpha\frac{(1-2p)^2}{np(1-p)}$. Thus, from composition, the output of the $s$ shufflers is $\left(\alpha,\epsilon\left(\alpha\right)\right)$-RDP, where $\epsilon\left(\alpha\right)$ is bounded by:
\begin{equation}~\label{eqn:rdp_binary_vector}
 \epsilon\left(\alpha\right)\leq c\alpha\frac{s(1-2p)^2}{np(1-p)}.
\end{equation} 
Observe that ~\eqref{eqn:rdp_binary_vector} gives a closed form bound on the RDP of the mechanism $\mathcal{R}^{\text{Bin}}_{p,s}$ in the shuffle model. However, we can numerically provide better bound on the RDP of the shuffle  model using~\cite{feldman2022stronger}. Now, we use Lemmas~\ref{lem:RDP_DP} to convert from RDP to central DP. For given $\delta>0$, shuffling the outputs of $n$ mechanisms $\mathcal{R}^{\text{Bin}}_{p,s}$ is $\left(\epsilon,\delta\right)$-DP, where $\epsilon$ is bounded by
\begin{equation}~\label{eqn:eps_delta_binary}
\epsilon\leq 2\sqrt{\frac{s(1-2p)^2\log(1/\delta)}{np(1-p)}}.
\end{equation}
By setting $p=\frac{1}{2}\left(1-\sqrt{\frac{v^2}{v^2+4}}\right)$ and $v^2=\frac{n\epsilon^2}{4s\log(1/\delta)}$, we can easily show that~\eqref{eqn:eps_delta_binary} is satisfied, and hence, the output of the shufflers is $\left(\epsilon,\delta\right)$-DP.  

\subsection{MSE bound of the local DP model (Theorem~\ref{thm:binary_vector_ldp}) and shuffle model (Theorem~\ref{thm:binary_vector_shuffle})}

For ease of analysis, we assume in the remaining part that $\frac{d}{s}$ is integer, otherwise, we can add dummy $s\lceil \frac{d}{s}\rceil-d$ zeros to the vector $\mathbf{b}_i$ to make the size of the vector divisible by $s$.

Now, we show that the output of the mechanism $\mathcal{R}^{\text{Bin}}_{p,s}$ is unbiased estimate of $\mathbf{b}_i$. Let $\mathcal{Y}_i$ be the output of Algorithm~\ref{algo:binary_vector} and $a=\frac{d}{s}$. We can represent the output $\mathcal{Y}_i$ as a vector of dimension $d$ that has $s$ non-zero elements $\mathbf{y}_i = [\mathbf{y}_{i1},\ldots,\mathbf{y}_{is}]$, where $\mathbf{y}_{ij}=a\mathcal{R}_p^{\textsl{2RR}}\left(\mathbf{b}_i[a_{ij}]\right)\mathbf{e}_{a_{ij}}$ is a sub-vector of dimensions $a$ that has only one non-zero element. Then, we have 
\begin{equation}
\begin{aligned}
\mathbb{E}\left[\mathbf{y}_{ij}\right]&=\frac{1}{a} \sum_{a_{ij}=(j-1)a+1}^{ja} a\mathbf{e}_{a_{ij}}\mathbb{E}\left[\mathcal{R}_p^{\textsl{2RR}}\left(\mathbf{b}_i[a_{ij}]\right)\right]\\
&\stackrel{\text{(a)}}{=}\sum_{a_{ij}=(j-1)a+1}^{ja}\mathbf{e}_{a_{ij}}\mathbf{b}_i[a_{ij}]\\
&=\mathbf{b}_i[(j-1)a+1:ja],
\end{aligned}
\end{equation}
where $\mathbf{e}_j$ denotes the $j$th basis vector and (a) follows from the fact that the mechanism $\mathcal{R}_p^{\textsl{2RR}}$ shown in Theorem~\ref{thm:binary_mse} is unbiased. $\mathbf{b}_i\left[l:m\right]$ denotes the values of the coordinates $l,l+1,\ldots,m$. As a result, we have that $\mathbb{E}\left[\mathbf{y}_i\right] = [\mathbb{E}\left[\mathbf{y}_{i1}\right],\ldots,\mathbb{E}\left[\mathbf{y}_{is}\right]]=\mathbf{b}_i$.
Hence, Algorithm~\ref{algo:binary_vector} is an unbiased estimate of the input $\mathbf{b}_i$. Furthermore, the variance of Algorithm~\ref{algo:binary_vector} is bounded by:
\begin{equation}~\label{eqn:var_bound_R}
\begin{aligned}
\mathbb{E}\left[\|\mathbf{y}_i-\mathbf{b}_i\|^2_2\right]&=\sum_{j=1}^{s}\mathbb{E}\left[\|\mathbf{y}_{ij}-\mathbf{b}_i[(j-1)a+1:ja]\|^2_2\right]\\
&=\sum_{j=1}^{s}\frac{1}{a}\sum_{a_{ij}=(j-1)a+1}^{ja}\mathbb{E}\left[\|a\mathbf{e}_{a_{ij}} \mathcal{R}_p^{\textsl{2RR}}\left(\mathbf{b}_i[a_{ij}]\right)-\mathbf{b}_i[(j-1)a+1:ja]\|^2\right]\\
&=\frac{1}{a}\sum_{j=1}^{s}\sum_{a_{ij}=(j-1)a+1}^{ja}\mathbb{E}\left[\|\mathbf{e}_{a_{ij}} a\mathcal{R}_p^{\textsl{2RR}}\left(\mathbf{b}_i[a_{ij}]\right)-\mathbf{e}_{a_{ij}} a\mathbf{b}_i[a_{ij}]+\mathbf{e}_{a_{ij}} a\mathbf{b}_i[a_{ij}]-\mathbf{b}_i[(j-1)a+1:ja]\|^2\right]\\
&\stackrel{\text{(a)}}{=}\frac{1}{a}\sum_{j=1}^{s}\sum_{a_{ij}=(j-1)a+1}^{ja}\mathbb{E}\left[\|\mathbf{e}_{a_{ij}} a\mathcal{R}_p^{\text{Bin}}\left(\mathbf{b}_i[a_{ij}]\right)-\mathbf{e}_{a_{ij}} a\mathbf{b}_i[a_{ij}]\|^2\right]\\
&\qquad+\|\mathbf{e}_{a_{ij}} a\mathbf{b}_i[a_{ij}]-\mathbf{b}_i[(j-1)a+1:ja]\|^2\\
&\stackrel{\text{(b)}}{=}\frac{sa^2p(1-p)}{(1-2p)^2}+\frac{1}{a}\sum_{j=1}^{d}\left((a-1)^2+(a-1)\right)\mathbf{b}_{i}^2[j]\\
&=\frac{sa^2p(1-p)}{(1-2p)^2}+\frac{(a-1)\left((a-1)+1\right)}{a}\sum_{j=1}^{d}\mathbf{b}_{i}^2[j]\\
&=\frac{a^2 s p(1-p)}{(1-2p)^2}+(a-1)\|\mathbf{b}_{i}\|^2\\
&\stackrel{\text{(c)}}{\leq} \frac{sa^2p(1-p)}{(1-2p)^2}+(a-1)d,
\end{aligned}
\end{equation}
where (a) follows from the fact that the 2RR mechanism $\mathcal{R}_p^{\textsl{2RR}}$ is unbiased and (b) from the variance of the 2RR mechanism $\mathcal{R}_p^{\textsl{2RR}}$ (see Theorem~\ref{thm:binary_mse}). Step (c) follows from the fact that $\|\mathbf{b}_i\|^2\leq d$. Hence, we can bound the MSE in the local DP model and the shuffle model as follows.

\textbf{MSE for the local DP model (Theorem~\ref{thm:binary_vector_ldp}):}
Observe that the output of the server $\hat{b}=\mathcal{A}^{\text{Bin}}\left(\mathcal{Y}_1,\ldots,\mathcal{Y}_n\right)$ can be represented as $\hat{b}=\frac{1}{n}\sum_{i=1}^{n}\mathbf{y}_i$, where $\mathbf{y}_i$ is the sparse representation of the $i$-th client private message discussed above. By setting $p=\frac{1}{2}\left(1-\sqrt{\frac{v^2}{v^2+4}}\right)$ and $v^2=\epsilon_0^2/s^2$, we have that:
\begin{equation}
\begin{aligned}
\mathsf{MSE}^{\text{Bin}}_{\text{ldp}}&=\sup_{\lbrace \mathbf{b}_i\in\lbrace 0,1\rbrace^d\rbrace}\mathbb{E}\left[\|\hat{\mathbf{b}}-\overline{\mathbf{b}}\|_2^2\right]\\
&\stackrel{\text{(a)}}{=}\frac{1}{n^2}\sum_{i=1}^{n}\mathbb{E}\left[\|\mathbf{y}_i-\mathbf{b}_i\|^2_2\right]\\
&\stackrel{\text{(b)}}{\leq} \frac{d(a-1)}{n}+ a^2\frac{sp(1-p)}{n(1-2p)^2}\\
&=\frac{d(\frac{d}{s}-1)}{n}+ d^2\frac{p(1-p)}{sn(1-2p)^2}\\
&\stackrel{\text{(c)}}{=}\frac{d^2}{n}\left( \left(\frac{1}{s}-\frac{1}{d}\right)+ \frac{s}{\epsilon_0^2}\right)\\
&=\mathcal{O}\left(\frac{d^2}{n}\max\left\{\frac{1}{s},\frac{s}{\epsilon_0^2}\right\}\right),
\end{aligned}
\end{equation}
where (a) follows from the i.i.d of the random mechanisms $\mathcal{R}^{\text{Bin}}_{p,s}$. Step (b) follows from the variance of the mechanism $\mathcal{R}^{\text{Bin}}_{p,s}$ in~\eqref{eqn:var_bound_R}. Step (c) follows from substituting $p=\frac{1}{2}\left(1-\sqrt{\frac{v^2}{v^2+4}}\right)$ and $v^2=\epsilon_0^2/s^2$. This completes the proof of Theorem~\ref{thm:binary_vector_ldp}.

\textbf{MSE for the MMS model (Theorem~\ref{thm:binary_vector_shuffle}):}
Observe that the output of the server $\hat{b}=\mathcal{A}^{\text{Bin}}\left(\mathcal{Y}_1,\ldots,\mathcal{Y}_n\right)$ can be represented as $\hat{b}=\frac{1}{n}\sum_{i=1}^{n}\mathbf{y}_i$, where $\mathbf{y}_i$ is the sparse representation of the $i$-th client private message discussed above. By setting $p=\frac{1}{2}\left(1-\sqrt{\frac{v^2}{v^2+4}}\right)$ and $v^2=\frac{n\epsilon^2}{4s\log(1/\delta)}$, we have that:
\begin{equation}
\begin{aligned}
\mathsf{MSE}^{\text{Bin}}_{\text{shuffle}}&=\sup_{\lbrace \mathbf{b}_i\in\lbrace 0,1\rbrace^d\rbrace}\mathbb{E}\left[\|\hat{\mathbf{b}}-\overline{\mathbf{b}}\|_2^2\right]\\
&\stackrel{\text{(a)}}{=}\frac{1}{n^2}\sum_{i=1}^{n}\mathbb{E}\left[\|\mathbf{y}_i-\mathbf{b}_i\|^2_2\right]\\
&\stackrel{\text{(b)}}{\leq} \frac{d(a-1)}{n}+ a^2\frac{sp(1-p)}{n(1-2p)^2}\\
&=\frac{d(\frac{d}{s}-1)}{n}+ d^2\frac{p(1-p)}{sn(1-2p)^2}\\
&\stackrel{\text{(c)}}{=}\frac{d^2}{n^2}\left( n\left(\frac{1}{s}-\frac{1}{d}\right)+ \frac{4\log(1/\delta)}{\epsilon^2}\right)\\
&=\mathcal{O}\left(\frac{d^2}{n^2}\max\left\{n\left(\frac{1}{s}-\frac{1}{d}\right),\frac{\log(1/\delta)}{\epsilon^2}\right\}\right),
\end{aligned}
\end{equation}
where (a) follows from the i.i.d of the random mechanisms $\mathcal{R}^{\text{Bin}}_{p,s}$. Step (b) follows from the variance of the mechanism $\mathcal{R}^{\text{Bin}}_{p,s}$ in~\eqref{eqn:var_bound_R}. Step (c) follows from substituting $p=\frac{1}{2}\left(1-\sqrt{\frac{v^2}{v^2+4}}\right)$ and $v^2=\frac{n\epsilon^2}{4s\log(1/\delta)}$. This completes the proof of Theorem~\ref{thm:binary_vector_shuffle}.

%% file: App_quantization.tex
\section{Properties of Quantization scheme}~\label{app:quantization}
In this section, we prove some properties of the quantization scheme proposed in Section~\ref{sec:l_inf_norm} for vector $\mathbf{z}_i\in[0,1]^d$. We first prove some properties for a scalar case when $x\in[0,1]$, and then, the results of the bounded $\ell_{\infty}$ will be obtained directly from repeating the scalar case on each coordinate.

Let $x\in[0,1]$ and $x^{(k)}=\sum_{l=1}^{s} b_{l} 2^{-l}$ for $k\geq 1$, where $x^{(0)}=0$ and $b_{k} = \lfloor 2^{k}(x-x^{k-1}) \rfloor$. For given $m\geq 1$, we represent $x$ using $m$ bits as follows: $\tilde{x}^{(m)}=\sum_{k=1}^{m-1}b_{k} 2^{-k}+u 2^{-m+1}$, where $u =\mathsf{Bern}\left(2^{m-1}(x-x^{(m-1)}[j])\right)$. This estimator needs only $m$ bits of representation.
\begin{lemma}~\label{lemm:quant_scalar} For given $x\in[0,1]$, let $\tilde{x}^{(m)}$ be the quantization of $x$ presented above. We have that $\tilde{x}^{(m)}$ is an unbiased estimate of $x$ with bounded MSE:
\begin{equation}
\mathsf{MSE}^{\text{quan}}_{\text{scalar}}=\sup_{x\in[0,1]}\mathbb{E}\left[\|x-\tilde{x}^{(m)}\|_2^2\right]\leq \frac{1}{4^{m}},
\end{equation}
where the expectation is taken over the randomness in the quantization scheme.
\end{lemma}

\begin{proof}
First, we show that $\tilde{x}^{(m)}$ is an unbiased estimate of $x$:
\begin{equation}
\begin{aligned}
\mathbb{E}\left[\tilde{x}^{m}\right]&=\sum_{k=1}^{m-1}b_{k}2^{-k}+\mathbb{E}\left[u\right]2^{-m+1}\\
&\stackrel{\text{(a)}}{=}\sum_{k=1}^{m-1}b_{k}2^{-k}+ 2^{m-1}(x-x^{(m-1)})2^{-m+1}\\
&=x_i,
\end{aligned}
\end{equation}
where step (a) is obtained from the fact that $u$ is a Bernoulli random variable with bias $p=2^{m-1}(x-x^{(m-1)})$. We show that the estimator $\tilde{x}^{(m)}$ has a bounded MSE by $4^{-m}$:
\begin{equation}~\label{eqn:non_private_variace}
\begin{aligned}
\mathsf{MSE}^{\text{quan}}_{\text{scalar}}&=\sup_{x\in[0,1]}\mathbb{E}\left[\|x-\tilde{x}^{(m)}\|_2^2\right]\\
&=\sup_{x\in[0,1]}\mathbb{E}\left[\|x-x^{(m-1)}-u2^{-m+1}\|^2\right]\\
&=\sup_{x\in[0,1]}4^{-(m-1)}\mathbb{E}\left[\|2^{-(m-1)}(x-x^{(m-1)})-u\|^2\right]\\
&\stackrel{\text{(a)}}{\leq}\frac{1}{4^{m}},
\end{aligned}
\end{equation}
where the inequality (a) is obtained from the fact that $u$ is a Bernoulli random variable and hence has a variance less that $1/4$. This completes the proof of Lemma~\ref{lemm:quant_scalar}. 
\end{proof}
\begin{corollary}~\label{cor:quant_vector} For given a vector $\mathbf{z}_i\in[0,1]^{d}$, let $\tilde{\mathbf{z}}_i^{(m)}$ be the quantization of $\mathbf{z}_i$ by applying the above scalar quantization scheme on each coordinate $\mathbf{z}_i[j]$ for $j\in[d]$. Then, $\tilde{\mathbf{z}}_i^{(m)}$ is an ubiased estimate of $\mathbf{z}_i$ with bounded MSE:
\begin{equation}
\mathsf{MSE}^{\text{quan}}_{\text{vector}}=\sup_{\mathbf{z}_i\in[0,1]^d}\mathbb{E}\left[\|\mathbf{z}_i-\tilde{\mathbf{z}}^{(m)}_i\|_2^2\right]\leq \frac{d}{4^{m}},
\end{equation}
where the expectation is taken over the randomness in the quantization scheme.
\end{corollary}

%% file: App_linf_norm.tex
\section{Proofs of Theorem~\ref{thm:l_inf_vector_ldp} and Theorem~\ref{thm:l_inf_vector_shuffle} (Bounded $\ell_{\infty}$-norm vectors)}~\label{app:l_inf_vector}

In this section, we prove Theorem~\ref{thm:l_inf_vector_ldp} and Theorem~\ref{thm:l_inf_vector_shuffle} for the mean of bounded $\ell_{\infty}$-norm vectors in local DP and shuffle models, respectively.
\subsection{Communication Bound for Theorem~\ref{thm:l_inf_vector_ldp} and Theorem~\ref{thm:l_inf_vector_shuffle}}

In the mechanism $\mathcal{R}^{\ell_{\infty}}_{v,m,s}$, the client sends $m$ binary vectors $\mathbf{b}_i^{(1)},\ldots,\mathbf{b}_i^{(m-1)},\mathbf{u}_i$ using the private mechanism $\mathcal{R}_{p,s}^{\text{Bin}}$. From Theorem~\ref{thm:binary_vector_ldp} and Theorem~\ref{thm:binary_vector_shuffle}, the private mechanism $\mathcal{R}_{p,s}^{\text{Bin}}$ needs   
$\log\left(\lceil \frac{d}{s}\rceil\right)+1$ bits for communication. Thus, the total communication of the private mechanism $\mathcal{R}^{\ell_{\infty}}_{v,m,s}$ is $ms\left(\log\left(\lceil \frac{d}{s}\rceil\right)+1\right)$-bits.

\subsection{Privacy of the local DP model in Theorem~\ref{thm:l_inf_vector_ldp}}
In the mechanism $\mathcal{R}^{\ell_{\infty}}_{v,p,s}$, each client sends $m$ messages from the private mechanism $\mathcal{R}_{p,s}^{\text{Bin}}$ as follows: $\mathcal{R}^{\text{Bin}}_{p_1,s}(\mathbf{b}_i^{(1)}),\ldots,\mathcal{R}^{\text{Bin}}_{p_{m-1},s}(\mathbf{b}_i^{(m-1)}),\mathcal{R}^{\text{Bin}}_{p_m,s}(\mathbf{u}_i)$, where $p_i=\frac{1}{2}\left(1\sqrt{\frac{v_k^2/s^2}{v_k^2/s^2+4}}\right)$ and $v_k=\frac{4^{\frac{-k}{3}}}{\left(\sum_{l=1}^{m-1}4^{\frac{-l}{3}}+4^{\frac{-m+1}{3}}\right)}v$ for $k\in[m-1]$ and $v_m=\frac{4^{\frac{-m+1}{3}}}{\left(\sum_{l=1}^{m-1}4^{\frac{-l}{3}}+4^{\frac{-m+1}{3}}\right)}v$. Hence, from Theorem~\ref{thm:binary_vector_ldp}, the $k$-th message $\mathcal{R}^{\text{Bin}}_{p_k,s}(\mathbf{b}_i^{(k)})$ is $\epsilon_0^{(k)}$-LDP, where $\epsilon_0^{(k)}=v_k$ for $k\in[m]$. As a results, the total mechanism $\mathcal{R}^{\ell_{\infty}}_{v,m,s}$ is bounded by:
\begin{equation}
\begin{aligned}
\epsilon_0&=\sum_{k=1}^{m}\epsilon_0^{(k)} =\sum_{k=1}^{m}v_{k}=\sum_{k=1}^{m-1}\left\{\frac{4^{\frac{-k}{3}}}{\left(\sum_{l=1}^{m-1}4^{\frac{-l}{3}}+4^{\frac{-m+1}{3}}\right)}v\right \}+\frac{4^{\frac{-m+1}{3}}}{\left(\sum_{l=1}^{m-1}4^{\frac{-l}{3}}+4^{\frac{-m+1}{3}}\right)}v=v,
\end{aligned}
\end{equation}
from the composition of the DP mechanisms~\cite{dwork2014algorithmic}, note that  we will choose $v=\epsilon_0$.

In addition, we can bound the RDP of the mechanism $\mathcal{R}^{\ell_{\infty}}_{v,m,s}$ in the local DP model by using the composition of the RDP (see Lemma~\ref{lemm:compostion_rdp}). From the proof of Theorem~\ref{thm:l_inf_vector_ldp} in Appendix~\ref{app:binary_vector}, the mechanism $\mathcal{R}_{p_k,s}^{\text{Bin}}$ is $\left(\alpha,\epsilon^{(k)}\left(\alpha\right)\right)$-RDP, where $\epsilon^{(k)}\left(\alpha\right)$ is bounded by:
\begin{equation}~\label{eqn:rdp_binary_bound}
\epsilon^{(k)}\left(\alpha\right)=\frac{s}{\alpha-1}\log\left(p^{\alpha}_k(1-p_k)^{1-\alpha}+p^{1-\alpha}_k(1-p_k)^{\alpha}\right),
\end{equation}
 Hence, the mechanism $\mathcal{R}^{\ell_{\infty}}_{v,m,s}$ is $\left(\alpha,\epsilon\left(\alpha\right)\right)$-RDP, where $\epsilon\left(\alpha\right)=\sum_{k=1}^{m}\epsilon^{(k)}\left(\alpha\right)$. 

\subsection{Privacy of the MMS model in Theorem~\ref{thm:l_inf_vector_shuffle}}
In the mechanism $\mathcal{R}^{\ell_{\infty}}_{v,p,s}$, each client sends $m$ messages from the private mechanism $\mathcal{R}_{p,s}^{\text{Bin}}$ as follows: $\mathcal{R}^{\text{Bin}}_{p_1,s}(\mathbf{b}_i^{(1)}),\ldots,\mathcal{R}^{\text{Bin}}_{p_{m-1},s}(\mathbf{b}_i^{(m-1)}),\mathcal{R}^{\text{Bin}}_{p_m,s}(\mathbf{u}_i)$, where $p_i=\frac{1}{2}\left(1\sqrt{\frac{v_k^2/s^2}{v_k^2/s^2+4}}\right)$ and $v_k=\frac{4^{\frac{-k}{3}}}{\left(\sum_{l=1}^{m-1}4^{\frac{-l}{3}}+4^{\frac{-m+1}{3}}\right)}v$ for $k\in[m-1]$ and $v_m=\frac{4^{\frac{-m+1}{3}}}{\left(\sum_{l=1}^{m-1}4^{\frac{-l}{3}}+4^{\frac{-m+1}{3}}\right)}v$.

From the proof of Theorem~\ref{thm:binary_vector_shuffle} in Appendix~\ref{app:binary_vector}, shuffling the outputs of $n$ mechanisms $\mathcal{R}^{\text{Bin}}_{p_k,s}$ is $\left(\alpha,\epsilon^{(k)}\left(\alpha\right)\right)$, where $\epsilon^{(k)}\left(\alpha\right)$ is bounded by:
\begin{equation}
\epsilon^{(k)}\left(\alpha\right)= \leq c\alpha\frac{s(1-2p_k)^2}{np_k(1-p_k)}=c\alpha\frac{v_k^2}{sn},
\end{equation}
from~\eqref{eqn:rdp_binary_vector}, where the last equality is obtained by substituting $p_k=\frac{1}{2}\left(1-\sqrt{\frac{v_k^2}{v_k^2+4}}\right)$. From Lemma~\ref{lemm:compostion_rdp} of RDP composition, we get that the total RDP of the mechanism $\mathcal{R}^{\ell_{\infty}}_{v,m,s}$ is bounded by:
\begin{equation}
\begin{aligned}
\epsilon\left(\alpha\right)&=\sum_{k=1}^{m}\epsilon^{(k)}\left(\alpha\right)
=c\frac{\alpha}{sn}\sum_{k=1}^{m}v_k^2=c\frac{\alpha v^2}{sn}\sum_{k=1}^{m}f_k^2
\leq c\frac{\alpha v^2}{sn}, 
\end{aligned}
\end{equation}
where $f_k=\frac{4^{\frac{-k}{3}}}{\left(\sum_{l=1}^{m-1}4^{\frac{-l}{3}}+4^{\frac{-m+1}{3}}\right)}$ for $k\in[m]$ and $f_m=\frac{4^{\frac{-m+1}{3}}}{\left(\sum_{l=1}^{m-1}4^{\frac{-l}{3}}+4^{\frac{-m+1}{3}}\right)}$. The last inequality is obtained from the fact that $\sum_{k=1}^{m} f_k=1$ and hecne $\sum_{k=1}^{m}f_k^2\leq 1$. Thus, we use Lemma~\ref{lem:RDP_DP} to convert from RDP to central DP. For given $\delta>0$, shuffling the outputs of $n$ mechanisms $\mathcal{R}^{\ell_{\infty}}_{v,m,s}$ is $\left(\epsilon,\delta\right)$-DP, where $\epsilon$ is bounded by
\begin{equation}~\label{eqn:app_eps_delta_l_inf}
\epsilon\leq 2\sqrt{\frac{v^2\log(1/\delta)}{sn}}.
\end{equation}
By setting $v^2=\frac{sn\epsilon^2}{4\log(1/\delta)}$, we can easily show that~\eqref{eqn:eps_delta_l_inf} is satisfied, and hence, the output of the shufflers is $\left(\epsilon,\delta\right)$-DP. 
 
\subsection{MSE bound of the local DP model (Theorem~\ref{thm:l_inf_vector_ldp}) and MMS model (Theorem~\ref{thm:l_inf_vector_shuffle})}
We first present some notations to simplify the analysis. For given $\mathbf{x}_i\in\mathbb{B}^{d}_{\infty}\left(r_{\infty}\right)$, we define $\mathbf{z}_i=\frac{\mathbf{x}_i+r_{\infty}}{2r_{\infty}}$, where the operations are done coordinate-wise. Thus, we have that $\mathbf{z}_{i}\in[0,1]^{d}$. For given $\mathbf{z}_{i}\in[0,1]^{d}$ and $m\geq 1$, we define $\tilde{\mathbf{z}}^{(m)}_i=\sum_{k=1}^{m-1}\mathbf{b}_{i}^{(k)}2^{-k}+\mathbf{u}_i2^{-m+1}$, where $\mathbf{b}_i^{(k)}=\lfloor 2^{k}\left(\mathbf{z}_i-\mathbf{z}^{(k-1)}_i\right)\rfloor$ and $\mathbf{z}_i^{(0)}=\mathbf{0}$ and $\mathbf{z}_i^{(k)}=\sum_{l=1}^{k}\mathbf{b}_i^{(l)}2^{-l}$ for $k\geq 1$. Furthermore, $\mathbf{u}_i$ is a Bernoulli vector defined by $\mathbf{u}_i=\mathsf{Bern}\left(2^{m-1}\left(\mathbf{z}_i-\mathbf{z}^{(m-1)}_i\right)\right)$. Let $\overline{\mathbf{b}}^{(k)}=\frac{1}{n}\sum_{i=1}^{n}\mathbf{b}_i^{(k)}$, $\overline{\mathbf{u}}=\frac{1}{n}\sum_{i=1}^{n}\mathbf{u}_i$, and $\overline{\tilde{\mathbf{z}}}^{(m)}=\frac{1}{n}\sum_{i=1}^{n}\tilde{\mathbf{z}}^{(m)}_i$.

\textbf{MSE for the local DP model (Theorem~\ref{thm:binary_vector_ldp}):}
Observe that the output of the server $\hat{\mathbf{x}}=\mathcal{A}^{\ell_{\infty}}\left(\mathcal{Y}_1,\ldots,\mathcal{Y}_n\right)=2r_{\infty}\hat{\mathbf{z}}-r_{\infty}$, where $\hat{\mathbf{z}}=\sum_{k=1}^{m-1}\hat{\mathbf{b}}^{(k)}+\hat{\mathbf{u}}2^{-m+1}$. Thus, we have that:
\begin{align}
\nonumber
\mathsf{MSE}^{\ell_{\infty}}_{\text{ldp}}&=\sup_{\lbrace \mathbf{x}_i\in\mathbb{B}^{d}_{\infty}\left(r_{\infty}\right)\rbrace}\mathbb{E}\left[\|\hat{\mathbf{x}}-\overline{\mathbf{x}}\|_2^2\right]\\ \nonumber
&\stackrel{\text{(a)}}{=}r_{\infty}^2\sup_{\lbrace \mathbf{z}_i\in[0,1]^d\rbrace}\mathbb{E}\left[\|\hat{\mathbf{z}}-\overline{\mathbf{z}}\|_2^2\right]\\ \nonumber
&=r_{\infty^2}\sup_{\lbrace \mathbf{z}_i\in[0,1]^d\rbrace}\mathbb{E}\left[\|\hat{\mathbf{z}}-\overline{\tilde{\mathbf{z}}}^{(m)}+\overline{\tilde{\mathbf{z}}}^{(m)}-\overline{\mathbf{z}}\|_2^2\right]\\ \nonumber
&\stackrel{\text{(b)}}{=}r_{\infty}^2\sup_{\lbrace \mathbf{z}_i\in[0,1]^d\rbrace}\left(\mathbb{E}\left[\|\hat{\mathbf{z}}-\overline{\tilde{\mathbf{z}}}^{(m)}\|_2^2\right]+\mathbb{E}\left[\|\overline{\tilde{\mathbf{z}}}^{(m)}-\overline{\mathbf{z}}\|_2^2\right]\right)\\ \nonumber
&\stackrel{\text{(c)}}{\leq}r_{\infty}^2\sup_{\lbrace \mathbf{z}_i\in[0,1]^d\rbrace}\left(\mathbb{E}\left[\|\sum_{k=1}^{m-1}\hat{\mathbf{b}}^{(k)}2^{-k}+\hat{\mathbf{u}}2^{-m+1}-\sum_{k=1}^{m-1}\overline{\mathbf{b}}^{(k)}2^{-k}+\overline{\mathbf{u}}2^{-m+1}\|_2^2\right]+\frac{d}{n4^{m}}\right)\\ \nonumber
&\stackrel{\text{(d)}}{\leq}r_{\infty}^2\left(\sum_{k=1}^{m-1}\frac{d^24^{-k}}{n}\left(\frac{1}{s}+\frac{s}{v_k^2}\right)+\frac{d^24^{-m+1}}{n}\left(\frac{1}{s}+\frac{s}{v_m^2}\right)+\frac{d}{n4^{m}}\right)\\ \nonumber
&\stackrel{\text{(e)}}{\leq}r_{\infty}^2\left(\frac{d^2}{ns}\left(\sum_{k=1}^{m-1}4^{-k}+4^{-m+1}\right)+\frac{d^2s}{nv^2}\left(\sum_{k=1}^{m-1}4^{-k/3}+4^{-(m-1)/3}\right)^{3}+\frac{d}{n4^{m}}\right)\\ \label{eq:ExpConst-MSE-linf-LDP}
&\stackrel{\text{(f)}}{\leq}r_{\infty}^2\left(\frac{3d^2}{ns}+\frac{5d^2s}{n\epsilon_0^2}+\frac{d}{n4^{m}}\right)\\ \label{eq:Order-MSE-linf-LDP}
&=\mathcal{O}\left(\frac{r_{\infty}^2d^2}{n}\max\left\{\frac{1}{d4^{m}},\frac{1}{s},\frac{s}{\epsilon_0^2}\right\}\right),
\end{align}
where (a) follows from the fact that $\mathbf{z}_i$ is a linear transformation of $\mathbf{x}_i$. Step (b) follows from the fact that $\overline{\tilde{\mathbf{z}}}^{(m)}$ is an unbiased estimate of $\overline{\mathbf{z}}$ from Corollary~\ref{cor:quant_vector}. Step (c) from the bound of the MSE of the quantization scheme $\overline{\tilde{\mathbf{z}}}^{(m)}$ in Corollary~\ref{cor:quant_vector}. Step (d) follows from the MSE of the private mean estimation of binary vectors in Theorem~\ref{thm:binary_vector_ldp}. Step (e) follows from substituting $v_k=\frac{4^{\frac{-k}{3}}}{\left(\sum_{l=1}^{m-1}4^{\frac{-l}{3}}+4^{\frac{-m+1}{3}}\right)}v$. Step (f) follows from the geometric series bound. This completes the proof of Theorem~\ref{thm:l_inf_vector_ldp}. 

\textbf{MSE for the MMS model (Theorem~\ref{thm:binary_vector_shuffle}):}
Observe that the output of the server $\hat{\mathbf{x}}=\mathcal{A}^{\ell_{\infty}}\left(\mathcal{Y}_1,\ldots,\mathcal{Y}_n\right)=2r_{\infty}\hat{\mathbf{z}}-r_{\infty}$, where $\hat{\mathbf{z}}=\sum_{k=1}^{m-1}\hat{\mathbf{b}}^{(k)}+\hat{\mathbf{u}}2^{-m+1}$. Thus, we have that:

\begin{align}
\nonumber
\mathsf{MSE}^{\ell_{\infty}}_{\text{shuffle}}&=\sup_{\lbrace \mathbf{x}_i\in\mathbb{B}^{d}_{\infty}\left(r_{\infty}\right)\rbrace}\mathbb{E}\left[\|\hat{\mathbf{x}}-\overline{\mathbf{x}}\|_2^2\right]\\ \nonumber
&\stackrel{\text{(a)}}{=}r_{\infty}^2\sup_{\lbrace \mathbf{z}_i\in[0,1]^d\rbrace}\mathbb{E}\left[\|\hat{\mathbf{z}}-\overline{\mathbf{z}}\|_2^2\right]\\ \nonumber
&=r_{\infty^2}\sup_{\lbrace \mathbf{z}_i\in[0,1]^d\rbrace}\mathbb{E}\left[\|\hat{\mathbf{z}}-\overline{\tilde{\mathbf{z}}}^{(m)}+\overline{\tilde{\mathbf{z}}}^{(m)}-\overline{\mathbf{z}}\|_2^2\right]\\ \nonumber
&\stackrel{\text{(b)}}{=}r_{\infty}^2\sup_{\lbrace \mathbf{z}_i\in[0,1]^d\rbrace}\left(\mathbb{E}\left[\|\hat{\mathbf{z}}-\overline{\tilde{\mathbf{z}}}^{(m)}\|_2^2\right]+\mathbb{E}\left[\|\overline{\tilde{\mathbf{z}}}^{(m)}-\overline{\mathbf{z}}\|_2^2\right]\right)\\ \nonumber
&\stackrel{\text{(c)}}{\leq}r_{\infty}^2\sup_{\lbrace \mathbf{z}_i\in[0,1]^d\rbrace}\left(\mathbb{E}\left[\|\sum_{k=1}^{m-1}\hat{\mathbf{b}}^{(k)}2^{-k}+\hat{\mathbf{u}}2^{-m+1}-\sum_{k=1}^{m-1}\overline{\mathbf{b}}^{(k)}2^{-k}+\overline{\mathbf{u}}2^{-m+1}\|_2^2\right]+\frac{d}{n4^{m}}\right)\\ \nonumber
&\stackrel{\text{(d)}}{\leq}r_{\infty}^2\left(\sum_{k=1}^{m-1}\frac{d^24^{-k}}{n}\left(\left(\frac{1}{s}-\frac{1}{d}\right)+\frac{s}{v_k^2}\right)+\frac{d^24^{-m+1}}{n}\left(\left(\frac{1}{s}-\frac{1}{d}\right)+\frac{s}{v_m^2}\right)+\frac{d}{n4^{m}}\right)\\ \nonumber
&\stackrel{\text{(e)}}{\leq}r_{\infty}^2\left(\frac{d^2}{n}\left(\frac{1}{s}-\frac{1}{d}\right)\left(\sum_{k=1}^{m-1}4^{-k}+4^{-m+1}\right)+\frac{d^2s}{nv^2}\left(\sum_{k=1}^{m-1}4^{-k/3}+4^{-(m-1)/3}\right)^{3}+\frac{d}{n4^{m}}\right)\\ \label{eq:ExpConst-MSE-lind-MMS}
&\stackrel{\text{(f)}}{\leq}r_{\infty}^2\left(\frac{3d^2}{n}\left(\frac{1}{s}-\frac{1}{d}\right)+\frac{5d^2\log\left(1/\delta\right)}{n^2\epsilon_0^2}+\frac{d}{n4^{m}}\right)\\ \label{eq:Order-MSE-linf-MMS}
&=\mathcal{O}\left(\frac{r_{\infty}^2d^2}{n^2}\max\left\{\frac{n}{d4^{m}},n\left(\frac{1}{s}-\frac{1}{d}\right),\frac{\log\left(1/\delta\right)}{\epsilon^2}\right\}\right),
\end{align}
where (a) follows from the fact that $\mathbf{z}_i$ is a linear transformation of $\mathbf{x}_i$. Step (b) follows from the fact that $\overline{\tilde{\mathbf{z}}}^{(m)}$ is an unbiased estimate of $\overline{\mathbf{z}}$ from Corollary~\ref{cor:quant_vector}. Step (c) from the bound of the MSE of the quantization scheme $\overline{\tilde{\mathbf{z}}}^{(m)}$ in Corollary~\ref{cor:quant_vector}. Step (d) follows from the MSE of the private mean estimation of binary vectors in Theorem~\ref{thm:binary_vector_shuffle}. Step (e) follows from substituting $v_k=\frac{4^{\frac{-k}{3}}}{\left(\sum_{l=1}^{m-1}4^{\frac{-l}{3}}+4^{\frac{-m+1}{3}}\right)}v$. Step (f) follows from the geometric series bound. This completes the proof of Theorem~\ref{thm:l_inf_vector_shuffle}.

%% file: App_l2_norm.tex
\section{Proofs of Theorem~\ref{thm:l_2_vector_ldp} and Theorem~\ref{thm:l_2_vector_shuffle} (Bounded $\ell_{2}$-norm vectors)}~\label{app:l_2_vector}

In this section, we prove Theorem~\ref{thm:l_2_vector_ldp} and Theorem~\ref{thm:l_2_vector_shuffle} for the mean of bounded $\ell_{2}$-norm vectors in local DP and shuffle models, respectively.

In the mechanism $\mathcal{R}^{\ell_2}_{v,m,s}$, each client applies random rotation to her vector $\mathbf{x}_i$ and then applies the private mechanism $\mathcal{R}^{\ell_{\infty}}_{v,m,s}$ to the bounded $\ell_{\infty}$-norm vector $\mathbf{w}_i$. Hence the communication and privacy are the same as the private mechanism $\mathcal{R}^{\ell_{\infty}}_{v,m,s}$. Thus, it remains to prove the MSE bound for both local DP model and shuffle model.

\subsection{MSE bound of the local DP model (Theorem~\ref{thm:l_2_vector_ldp}) and shuffle model (Theorem~\ref{thm:l_2_vector_shuffle})}
The proofs are obtained directly from the MSE of the bounded $\ell_{\infty}$-norm vector in Theorem~\ref{thm:l_inf_vector_ldp} and Theorem~\ref{thm:l_inf_vector_shuffle} with the following Theorem about the random rotation matrix.

\begin{theorem}~\label{thm:bounded_norm_2}~\cite{levy2021learning} Let $U=\frac{1}{\sqrt{d}}\mathbf{H}D$, where $\mathbf{H}$ denotes Hadamard matrix and $D$ is a diagonal matrix with i.i.d. uniformly ranodom $\lbrace \pm 1\rbrace$ entries. Let $\mathbf{x}_1,\ldots,\mathbf{x}_n\in\mathbb{B}_2^{d}\left(r_2\right)$ be bounded $\ell_2$-norm vectors and $\\mathbf{w}_i= U\mathbf{x}_i$. With probability at least $1-\beta$, we have that
\begin{equation}
\max_{i\in[n]}\|\mathbf{w}_i\|_{\infty}=\max_{i\in[n]}\|U\mathbf{x}_i\|_{\infty}\leq 10r_2 \sqrt{\frac{\log(\frac{nd}{\beta})}{d}}.
\end{equation}
\end{theorem}

From Lemma~\ref{thm:bounded_norm_2}, the vector $\mathbf{w}_i=U\mathbf{x}_i$ is bounded $\ell_{\infty}$-norm of radius $r_{\infty}=10r_2 \sqrt{\frac{\log(\frac{nd}{\beta})}{d}}$ with probability at least $1-\beta$. Hence, by plugging the radius $r_{\infty}=10r_2 \sqrt{\frac{\log(\frac{nd}{\beta})}{d}}$ into Theorem~\ref{thm:l_2_vector_ldp}, we obtained the MSE in Theorem~\ref{thm:l_2_vector_ldp}. Similarly, by plugging the radius $r_{\infty}=10r_2 \sqrt{\frac{\log(\frac{nd}{\beta})}{d}}$ into Theorem~\ref{thm:l_inf_vector_shuffle}, we obtained the MSE in Theorem~\ref{thm:l_2_vector_shuffle}.

\subsection{Lower bounds}
\label{subsec:LwrBnd}

A lower bound for local DP model was proposed in~\cite{chen2020breaking} in Theorem~$2.1$ stated.

\begin{theorem}[Lower Bound For local DP model~\cite{chen2020breaking}]~\label{thm:L_2_lower_bound_ldp} Let $n,d\in\mathbb{N}$ and $\epsilon_0>0$. For any $\mathbf{x}_1,\ldots,\mathbf{x}_n\in\mathbb{B}_2^{d}(r_2)$, the MSE is bounded below by:
 \begin{equation}
\mathsf{MSE}_{\text{LDP}}^{\ell_2}=\Omega\left(\frac{r_2^2d}{n\min\left\{\epsilon^2_0,\epsilon_0,b \right\}}\right)
\end{equation}
for any unbiased algorithm $\mathcal{M}$ that is $\epsilon_0$-LDP with $b$-bits of communication per client. 
\end{theorem}

Our lower bound for the shuffle model in Theorem~\ref{thm:L_2_lower_bound_central} is a combination of the lower bound on DME with communication constraints proposed in~\cite{pmlr-v162-chen22c} and the lower bound on DME with central $\left(\epsilon,\delta\right)$-DP constraints proposed in~\cite{bun2014fingerprinting}.

%% file: App_OptRes.tex
\section{Application to private stochastic optimization for federated learning} \label{app:OptRes}

In this section, we exploit our private mechanisms for DME to give convergence guarantees for DP-SGD algorithm. We consider a standard SGD algorithm, where the server initialize the model by choosing $\theta^{0}\in\mathcal{C}$. At the $t$-th iteration, the server chooses uniformly at random a subset of clients of size $k\in[n]$ and sends the current model $\theta^{t}$ to the sampled clients. Let $\mathcal{S}_t\subset[n]$ denotes the set of sampled clients at the $t$-th iteration. Each sampled client $i\in\mathcal{S}_t$ computes the local gradient $\nabla F_i\left(\theta^{t}\right)$. Then, the client applies our private $\mathcal{R}^{\ell_2}_{v,m,s}$ mechanism before sending it to the shufflers. The sever received the shuffled messages and aggregates the private gradients and updates the model as follows:
\begin{equation}
\theta^{t+1}=\theta^{t}-\eta g_t,
\end{equation} 
where $g_t=\mathcal{A}^{\ell_2}\left(\lbrace\mathcal{Y}_i:i\in\mathcal{S}_t\rbrace\right)$ denotes the private estimate of the true average gradients $h_t = \frac{1}{k}\sum_{i\in\mathcal{S}_t}\nabla F_i\left(\theta^{t}\right)$. We present a standard results for convergence of the SGD algorithm for smooth non-convex functions.

\begin{theorem}[SGD convergence~\cite{agarwal2018cpsgd}]~\label{thm:sgd_converge} Let $F$ be $L$-smooth and $\forall \mathbf{\theta}\|\nabla F\left(\theta\right)\|_2\leq D$. Let $\theta^{0}$ satisfies $F\left(\theta^{0}\right)-F\left(\theta^{*}\right)\leq D_F$. Let $\mathcal{R}$ be a private-compression scheme and $\eta=\min\left\{L^{-1},\sqrt{2D_F}\left(\sigma\sqrt{LT}\right)^{-1}\right\}$. Then after $T$ iterations, we get:
\begin{equation}
\mathbb{E}_{t\sim\text{Unif}\left(T\right)}\left[\nabla F\left(\theta^{t}\right)\right]\leq \frac{2D_F L}{T}+\frac{2\sqrt{2}\sqrt{LD_F}\sigma}{\sqrt{T}}+DB,
\end{equation}
where $\sigma^2=2\max_{1\leq t\leq T}2\mathbb{E}\left[h_t-\nabla F\left(\theta^{t}\right)\right]+2\max_{1\leq t\leq T}2\mathbb{E}\left[g_t-h_t\right]$ and $B=\max_{1\leq t\leq T}\|\mathbb{E}\left[g_t-h_t\right]\|_2$ denotes the maximum bias. $h_t$ is the stochastic gradient at the $t$th iteration and $g_t$ is the private-compressed gradient after applying the mechanisms $\mathcal{R}$. The expectation is taken with respect to the randomness of gradient and the private-compression mechanism $\mathcal{R}$.
\end{theorem}

The above theorem directly relates the MSE of the DME algorithm $\mathcal{R}^{\ell_2}_{v,m,s}$ to the convergence of the SGD algorithm. We use this theorem along with privacy amplification by sub-sampling and strong composition theorem to derive the convergence of the DP-SGD algorithm described above.

\begin{theorem}[DP-SGD convergence]~\label{thm:app_Opt} Let $F$ be $L$-smooth and $\forall \mathbf{\theta}\|\nabla F\left(\theta\right)\|_2\leq D$. Let $\theta^{0}$ satisfies $F\left(\theta^{0}\right)-F\left(\theta^{*}\right)\leq D_F$. Let $\mathcal{R}^{\ell_2}_{v,m,s}$ be our private-compression scheme and $\eta=\min\left\{L^{-1},\sqrt{2D_F}\left(\sigma\sqrt{LT}\right)^{-1}\right\}$. By choosing $v^2=\frac{k\tilde{\epsilon}^2}{s\log(kT/n\delta)}$ and $\tilde{\epsilon}=\frac{n\epsilon}{k\sqrt{T\log\left(2/\delta\right)}}$, then after $T$ iterations, the total algorithm is $\left(\epsilon,\delta\right)$-DP. Furthermore, we get:
\begin{equation}
\mathbb{E}_{t\sim\text{Unif}\left(T\right)}\left[\nabla F\left(\theta^{t}\right)\right]\leq \mathcal{O}\left(\frac{L\sqrt{d D_F\log\left(2n/\delta\right)}}{n\epsilon}\right)
\end{equation}
\end{theorem}
\begin{proof}
\textbf{Privacy analysis:} Let $q=k/n$ denote the sampling ratio at each iteration. At each iteration, we apply our private mechanism $\mathcal{R}^{\ell_2}_{v,m,s}$ with parameter $v^2=\frac{k\tilde{\epsilon}^2}{s\log(kT/n\delta)}$. Hence, from Theorem~\ref{thm:l_2_vector_shuffle}, the output of the shuffling at the $t$-th iteration is $\left(\tilde{\epsilon},\frac{n\delta}{kT}\right)$-DP. since, we sample $k$ out of $n$ clients at the $t$-th round, then we get that the privacy budget of the $t$-th iteration is $\left(\epsilon_t,\frac{\delta}{T}\right)$-DP from privacy amplification by sub-sampling~\cite{Jonathan2017sampling}, where $\epsilon_t=\log\left(1+q\left(e^{\tilde{\epsilon}}-1\right)\right)$. Note that $\epsilon_t=\mathcal{O}\left(q\tilde{\epsilon}\right)$ when $\tilde{\epsilon}\leq 1$. Now by using the strong composition theorem, get that our mechanism is $\left(\epsilon,\delta\right)$-DP, where $\epsilon$ is bounded by:
\begin{equation}
\epsilon\leq\sqrt{2T\log\left(2/\delta\right)}\epsilon_t+T\epsilon_t\left(e^{\epsilon_t}-1\right).
\end{equation}
Observe that $\epsilon=\mathcal{O}\left(\sqrt{T\log\left(2/\delta\right)}\epsilon_t\right)$when $\epsilon_t\leq \frac{1}{\sqrt{T}}$. By choosing $\tilde{\epsilon}=\frac{n\epsilon}{k\sqrt{T\log\left(2/\delta\right)}}$, we guarantees. This completes the proof of the privacy analysis.

\textbf{Convergence analysis:} The convergence analysis follows directly from the MSE bound of our mechanism $\mathcal{R}^{\ell_2}_{v,m,s}$ in Theorem~\ref{thm:l_2_vector_shuffle} and the convergence of the standard SGD algorithm in Theorem~\ref{thm:sgd_converge}.  
\end{proof}
Note that in our DP-SGD algorithm, we assume that each client compute the full gradient $\nabla F_i\left(\theta^{t}\right)$ and then applies the private-compression mechanism $\mathcal{R}^{\ell_2}_{v,m,s}$. 